\documentclass{iopjournal}
\usepackage{lmodern}
\usepackage{amsmath,amsfonts,amssymb}
\usepackage{amsthm}
\usepackage{ragged2e}
\usepackage{array}
\usepackage{textcomp}
\usepackage{url}
\usepackage{graphicx}
\usepackage{booktabs}
\usepackage{bm}
\usepackage{mathtools}
\IfFileExists{xcolor.sty}{\usepackage{xcolor}}{\usepackage{color}}
\providecolor{blue}{rgb}{0,0,1}  
\definecolor{revblue}{rgb}{0.0, 0.35, 0.7}  
\newif\ifrev \revfalse
\long\def\rev#1{\ifrev\begingroup\color{revblue}#1\endgroup\else#1\fi}
\usepackage{pgfplots}
\usepackage{tikz}
\usepgfplotslibrary{fillbetween}
\usetikzlibrary{patterns,arrows.meta,calc,positioning}

\definecolor{colExp}{RGB}{0,114,178}       
\definecolor{colPow}{RGB}{213,94,0}        
\definecolor{colStr}{RGB}{0,158,115}       
\definecolor{colFleet}{RGB}{170,68,153}    
\definecolor{colData}{RGB}{50,50,50}       

\newtheorem{theorem}{Theorem}
\newtheorem{corollary}{Corollary}

\newtheorem{proposition}{Proposition}
\newtheorem{definition}{Definition}
\newtheorem{remark}{Remark}
\newtheorem{assumption}{Assumption}

\renewenvironment{abstract}{%
      \vspace{16pt plus3pt minus3pt}
      {\color{gray}\hrule} \ \\
      \noindent \fontsize{11}{12}\selectfont {\bfseries Abstract}\\
      \rm\ignorespaces \justifying}{\vspace{3mm} {\color{gray}\hrule}}

\begin{document}
\justifying

\articletype{Paper}

\title{Generalization bounds and sample complexity for remaining useful life prediction from complete degradation trajectories}

\author{Huy Hoang Le$^1$\orcid{0009-0009-2138-5468} and Kim-Anh Nguyen$^{2,*}$\orcid{0000-0003-3408-847X}}

\affil{$^1$R\&D Department, PowerMore Ltd., Da Nang 550000, Vietnam}

\affil{$^2$Faculty of Electrical Engineering, University of Science and Technology, The University of Danang, Da Nang 550000, Vietnam}

\affil{$^*$Author to whom any correspondence should be addressed.}

\email{nkanh@dut.udn.vn}


\begin{abstract}
Data-driven remaining useful life (RUL) prediction requires complete degradation trajectories for training, yet such run-to-failure data are scarce and expensive. Practitioners currently lack principled guidance on how many failure examples suffice for a given model and accuracy target. This paper develops a sample complexity framework for RUL prediction comprising seven main results organised around three themes. \rev{First, we establish fundamental learning rates: a distribution-free generalization bound shows that the uniform deviation of the mean squared error decreases as $O(B^{2}\sqrt{p/n})$, where $p$ is the model complexity and $n$ the number of trajectories, and a minimax lower bound proves that the $\Theta(p/n)$ rate is unimprovable.} \rev{Second, we quantify how domain knowledge accelerates learning: incorporating degradation physics reduces data requirements by up to two orders of magnitude for deep networks, a Bernstein-type analysis achieves the minimax-optimal $O(p/n)$ rate under high signal-to-noise conditions, and closed-form penalties reveal when an incorrectly assumed physics model hurts rather than helps.} \rev{Third, we characterise the impact of data quality: fleet variability induces an irreducible bias--variance tradeoff, while right-censored observations suffer an efficiency loss that depends critically on the degradation class.} Closed-form expressions are provided for exponential, power-law, and stretched-exponential degradation. \rev{Cross-domain validation against published turbofan, battery, and bearing benchmarks confirms the theoretical predictions within a factor of 2--3 on average.} The results yield practical guidelines for planning data collection, selecting model complexity, and evaluating physics model assumptions in prognostics applications.
\end{abstract}

\section{Introduction}
\label{sec:introduction}
\rev{Prognostic measurement, the estimation of a system's remaining useful life (RUL) from in-service sensor data, is a cornerstone of condition-based maintenance and a growing area of measurement science~\mbox{\cite{nguyen2014cbm, lei2018review, chen2024mstsensor}}.} \rev{Like any measurement task, its reliability depends on calibration data: in this case, complete degradation trajectories (CDTs) from equipment that has been run to failure. A fundamental metrology question therefore arises: \emph{how many run-to-failure measurements are needed to guarantee that a trained RUL predictor generalises to unseen units?}}

\rev{This question has acute practical significance in measurement planning and experimental design.} CDT data are expensive to obtain: accelerated life tests for turbofan engines cost millions of dollars, battery cycling to end-of-life takes months to years~\cite{severson2019battery}, and bearing degradation tests require dedicated test rigs~\cite{nectoux2012pronostia}. Consequently, benchmark datasets are small: the C-MAPSS simulation~\cite{saxena2008cmapss} contains 100--249 trajectories per subset, and the Severson battery dataset has 124 cells~\cite{severson2019battery}. Practitioners must decide, before investing in data collection, whether the planned sample size will suffice.

\rev{The dominant paradigm trains deep neural networks on CDTs: multivariate sensor recordings from equipment that has been operated until failure~\mbox{\cite{zheng2017lstm, li2018cnn, zhang2022transformer, mo2023transformer, ding2025foundation}}. While substantial progress has been made on model architectures and uncertainty quantification~\mbox{\cite{wu2024survey, li2024review, zhu2024bayesrul, chen2024mstsensor}}, a principled theory connecting sample size to prediction reliability has been absent.}

Statistical learning theory provides tools to answer such questions through sample complexity bounds~\cite{vapnik1998statistical, shalev2014understanding}. However, standard results assume independent and identically distributed (i.i.d.) data points, whereas RUL prediction involves \emph{trajectories}: each CDT is a time series $(X_i(t), T_i)$ where $T_i$ is the failure time. Moreover, the RUL prediction problem has unique structure: the target $R_i(t) = T_i - t$ is the RUL of unit $i$ at time $t$, deterministically linked to the failure time, degradation processes follow known physics, and industrial data are often right-censored (units removed from service before failure).

Despite the importance of the question, no prior work has developed sample complexity theory specific to RUL prediction. Existing generalization analyses for time series focus on forecasting~\cite{kuznetsov2017generalization} or general supervised learning with dependent data~\cite{mohri2012new}, not on the first-passage time structure of RUL. Physics-informed learning theory~\cite{raissi2019physics} addresses differential equation constraints but not the sample efficiency gains from degradation physics knowledge.

\rev{This paper fills this gap. Specifically, we address the following questions:}
\begin{enumerate}
\rev{\item \emph{Fundamental rates:} What is the minimum number of CDTs needed to guarantee a given prediction accuracy, and can this rate be improved?}
\rev{\item \emph{Physics leverage:} By how much does incorporating known degradation physics reduce data requirements, and what happens when the assumed physics is wrong?}
\rev{\item \emph{Data quality:} How do fleet variability, observation noise, and right-censoring affect the sample efficiency?}
\end{enumerate}

\rev{Our contributions, organised around these three questions, are as follows:}
\begin{enumerate}
\item A distribution-free generalization bound for RUL prediction that relates prediction error to the number of CDTs, model complexity, and degradation properties (Theorem~\ref{thm:sample_complexity}), with a minimax lower bound proving that the $\Theta(p/n)$ rate cannot be improved (Theorem~\ref{thm:minimax}).
\item A quantification of the sample complexity reduction achieved by physics-informed hypothesis classes (Theorem~\ref{thm:physics_reduction}), and an analysis of the penalty when the physics model is misspecified (Theorem~\ref{thm:misspecification}).
\item Fast convergence rates under low-noise (high signal-to-noise ratio (SNR)) degradation conditions via a Bernstein analysis, achieving the minimax-optimal $O(p/n)$ rate and providing explicit characterization of the noise exponent for each degradation class (Theorem~\ref{thm:fast_rates}).
\item An analysis of the bias--variance tradeoff induced by fleet variability in degradation parameters (Theorem~\ref{thm:fleet}).
\item Information-theoretic bounds on learning from right-censored data (Theorem~\ref{thm:censoring}).
\item Closed-form expressions for exponential, power-law, and stretched-exponential degradation classes across all results where elementary solutions exist; remaining cases (stretched-exponential censoring efficiency) reduce to one-dimensional quadratures.
\end{enumerate}

The remainder of the paper is organized as follows. Section~\ref{sec:related} reviews related work. Section~\ref{sec:formulation} formalizes the problem. Section~\ref{sec:results} presents the main results. Section~\ref{sec:validation} provides numerical validation. Section~\ref{sec:discussion} discusses practical guidelines, and Section~\ref{sec:conclusion} concludes.

\section{Related Work}
\label{sec:related}

\subsection{RUL Prediction Methods}

Deep learning for RUL prediction has progressed rapidly, evolving from early recurrent networks~\cite{heimes2008rnn, zheng2017lstm} through convolutional architectures~\cite{li2018cnn} to the recent adoption of attention-based and transformer models~\cite{zhang2022transformer, mo2023transformer}. \rev{More recently, foundation models and pre-trained representations have been explored for prognostics~\mbox{\cite{ding2025foundation}}, and Bayesian deep learning methods have improved uncertainty quantification~\mbox{\cite{zhu2024bayesrul}}.} Recent surveys~\rev{\cite{wu2024survey, li2024review, liu2025mstreview}} catalog an ever-growing library of architectural variants. \rev{Yet these surveys consistently highlight a persistent gap: while empirical results accumulate, the field lacks a theoretical understanding of \emph{when} a given architecture has enough data to generalise reliably.}

\subsection{Statistical Learning Theory}

The natural toolkit for such questions is statistical learning theory. Vapnik--Chervonenkis (VC) theory~\cite{vapnik1998statistical} and Rademacher complexity bounds~\cite{bartlett2002rademacher} provide distribution-free generalization guarantees for supervised learning with i.i.d.\ data. Extensions to dependent observations include mixing-based bounds~\cite{mohri2012new} and discrepancy-based approaches for non-stationary processes~\cite{kuznetsov2017generalization}. \rev{Recent work has also provided tighter generalisation bounds for deep networks via PAC-Bayes~\mbox{\cite{lotfi2024pacbayes}} and compression-based arguments~\mbox{\cite{arora2018compression}}.} However, none of these frameworks exploit the specific structure of RUL problems (monotonic degradation dynamics, first-passage time targets, and the availability of known degradation physics), leaving a gap between general-purpose theory and the needs of prognostics practitioners.

\subsection{Physics-Informed Learning}

A parallel line of research embeds physical knowledge directly into the learning process. Physics-informed neural networks~\cite{raissi2019physics} enforce differential equation constraints during training, and their application to prognostics~\cite{wang2024pinn, li2024review, zhao2026pirnn} has shown improved accuracy, especially when training data are scarce. \rev{Transfer learning and domain adaptation have also been applied to prognostics to mitigate data scarcity across operating conditions~\mbox{\cite{da2024transfer, ragab2023contrastive}}.} Despite these encouraging results, the theoretical basis for the sample efficiency gains from physics incorporation remains informal: existing analyses demonstrate empirical benefit but do not quantify \emph{how much} data are saved or under what conditions a misspecified physics model can hurt rather than help.

\subsection{Sample Size for Prognostics}

The question of how many failure trajectories are needed for reliable RUL prediction has received only empirical attention. Benchmark studies~\cite{ramasso2014benchmark, heimes2008rnn} have documented the effect of training set size on prediction accuracy, but no prior work provides a theoretical framework that predicts the minimum sample size from first principles. Our work fills this gap by developing the first sample complexity theory tailored to the CDT-based RUL prediction setting.

\section{Problem Formulation}
\label{sec:formulation}

\subsection{Degradation Model and CDT Data}

Consider a fleet of $n$ units, each instrumented with $d$ sensors. Unit~$i$ ($i = 1, \ldots, n$) follows a degradation process:
\begin{equation}
\label{eq:degradation}
X_i(t) = D(\tau_i; \bm{\theta}_i) + \sigma W_i(t), \quad t \in [0, T_i],
\end{equation}
where $D(\tau_i; \bm{\theta_i})$ is the deterministic degradation function with normalized time $\tau_i = t/T_i \in [0,1]$, $\bm{\theta}_i$ are unit-specific physics parameters drawn from a fleet distribution $\pi(\bm{\theta})$, $\sigma$ is the noise level, and $W_i(t)$ is a standard Wiener process. The failure time $T_i$ is determined by the first-passage condition $D(1; \bm{\theta}_i) = 1$ (the degradation function is normalized so that $D(0;\bm{\theta}) = 0$ and $D(1;\bm{\theta}) = 1$ for all $\bm{\theta}$, absorbing the physical failure threshold into the normalization).

\begin{remark}
Equation~\eqref{eq:degradation} describes a scalar health indicator. In practice, unit~$i$ produces a $d$-dimensional sensor signal $\bm{X}_i(t) \in \mathbb{R}^d$; the scalar degradation state $D(\tau)$ can be viewed as a latent variable that drives all $d$ channels. The hypothesis class defined below takes the full $d$-dimensional observation as input, and the theoretical results (Theorems~\ref{thm:sample_complexity}--\ref{thm:misspecification}) hold for this multivariate setting since the key quantity, the per-trajectory loss, remains scalar and bounded.
\end{remark}

The SNR of the degradation process is defined as
\begin{equation}
\label{eq:snr}
\Gamma = \frac{\int_0^1 (dD/d\tau)^2\,d\tau}{\sigma^2},
\end{equation}
measuring the strength of the deterministic degradation signal relative to the measurement noise (the numerator is the $L^2[0,1]$ norm squared of the degradation rate). This quantity governs the convergence rate (Theorem~\ref{thm:fast_rates}) and the censoring efficiency (Theorem~\ref{thm:censoring}).

The RUL at time $t < T_i$ is $R_i(t) = T_i - t$. A CDT dataset consists of $n$ complete trajectories:
\begin{equation}
\label{eq:dataset}
\mathcal{S}_n = \{(X_i(\cdot), T_i)\}_{i=1}^n.
\end{equation}

\subsection{Degradation Model Classes}
\label{sec:deg_classes}

We consider three classes spanning major industrial domains:

\emph{Exponential degradation} (thermal fatigue~\cite{bayerer2008igbt}, chemical aging~\cite{vetter2005battery}):
\begin{equation}
\label{eq:exp_deg}
D(\tau; \alpha) = \frac{1 - e^{-\alpha\tau}}{1 - e^{-\alpha}}, \quad \alpha > 0.
\end{equation}

\emph{Power-law degradation} (crack growth~\cite{paris1963crack}):
\begin{equation}
\label{eq:pow_deg}
D(\tau; \beta) = \tau^\beta, \quad \beta > 1.
\end{equation}

\emph{Stretched-exponential degradation} (battery capacity fade~\cite{vetter2005battery, attia2022knee}):
\begin{equation}
\label{eq:str_deg}
D(\tau; \alpha_b, b) = \frac{1 - e^{-\alpha_b \tau^b}}{1 - e^{-\alpha_b}}, \quad \alpha_b > 0,\; b \in (0, 1].
\end{equation}

Fig.~\ref{fig:degradation_classes} illustrates the three classes and their corresponding RUL profiles.

\begin{figure}[!t]
\centering
\begin{tikzpicture}
\begin{axis}[
    name=panelA,
    at={(-0.8cm,0)}, anchor=outer east,
    width=0.36\textwidth, height=5.5cm,
    xlabel={Normalized time $\tau = t/T$},
    ylabel={Degradation $D(\tau)$},
    xmin=0, xmax=1.05, ymin=-0.02, ymax=1.08,
    xtick={0, 0.2, 0.4, 0.6, 0.8, 1.0},
    ytick={0, 0.2, 0.4, 0.6, 0.8, 1.0},
    grid=both, grid style={gray!15, thin},
    tick label style={font=\footnotesize},
    label style={font=\small},
    legend to name=fig1legend,
    legend style={
        font=\scriptsize,
        draw=none,
        fill=none,
        row sep=1pt,
        legend columns=4,
        /tikz/every even column/.append style={column sep=6pt},
    },
    every axis plot/.append style={line width=1.5pt},
    title={\small\textbf{(a)} Degradation classes},
    title style={at={(0.5,1.02)}, anchor=south},
]
\addplot[gray!50, dashed, thin, domain=0:1, samples=50] {x};
\addlegendentry{Linear}
\addplot[colExp, solid, domain=0:1, samples=200,
    mark=o, mark repeat=25, mark size=1.8pt, mark options={solid, fill=white, colExp}]
    {(1-exp(-2*x))/(1-exp(-2))};
\addlegendentry{Exp.\ ($\alpha\!=\!2$)}
\addplot[colStr, densely dashdotted, domain=0:1, samples=200,
    mark=triangle*, mark repeat=25, mark size=2pt, mark options={solid, colStr}]
    {(1-exp(-3*x^0.7))/(1-exp(-3))};
\addlegendentry{Str.-exp.\ ($b\!=\!0.7$)}
\addplot[colPow, densely dashed, domain=0:1, samples=200,
    mark=square*, mark repeat=25, mark size=1.8pt, mark options={solid, colPow}]
    {x^3};
\addlegendentry{Power-law ($\beta\!=\!3$)}
\end{axis}

\begin{axis}[
    name=panelB,
    at={(0.8cm,0)}, anchor=outer west,
    width=0.36\textwidth, height=5.5cm,
    xlabel={Normalized time $\tau = t/T$},
    ylabel={Misspecification residual},
    xmin=0, xmax=1.05, ymin=-0.01, ymax=0.50,
    xtick={0, 0.2, 0.4, 0.6, 0.8, 1.0},
    ytick={0, 0.1, 0.2, 0.3, 0.4, 0.5},
    grid=both, grid style={gray!15, thin},
    tick label style={font=\footnotesize},
    label style={font=\small},
    every axis plot/.append style={line width=1.3pt},
    title={\small\textbf{(b)} $|D(\tau) - \tau|$: penalty for assuming linear},
    title style={at={(0.5,1.02)}, anchor=south, align=center},
]
\addplot[name path=zero_b, draw=none, domain=0:1, samples=50, forget plot] {0};
\addplot[name path=exp_r, colExp, solid, domain=0:1, samples=200,
    mark=o, mark repeat=25, mark size=1.8pt, mark options={solid, fill=white, colExp},
    forget plot]
    {abs( (1-exp(-2*x))/(1-exp(-2)) - x )};
\addplot[colExp!15, forget plot] fill between[of=exp_r and zero_b];
\addplot[name path=pow_r, colPow, densely dashed, domain=0:1, samples=200,
    mark=square*, mark repeat=25, mark size=1.8pt, mark options={solid, colPow},
    forget plot]
    {abs( x^3 - x )};
\addplot[pattern=north east lines, pattern color=colPow!35, forget plot] fill between[of=pow_r and zero_b];
\addplot[colStr, densely dashdotted, domain=0:1, samples=200,
    mark=triangle*, mark repeat=25, mark size=2pt, mark options={solid, colStr},
    forget plot]
    {abs( (1-exp(-3*x^0.7))/(1-exp(-3)) - x )};
\end{axis}
\node[anchor=north] at ($(panelA.south)!0.5!(panelB.south)+(0,-1.3cm)$) {\ref{fig1legend}};
\end{tikzpicture}
\caption{%
\textbf{(a)}~Three degradation classes normalized to $D(0)\!=\!0$, $D(1)\!=\!1$. Exponential concentrates degradation rate early, power-law late; this determines censoring efficiency (Theorem~\ref{thm:censoring}).
\textbf{(b)}~Misspecification residual $|D(\tau)-\tau|$ when a linear model is assumed. The shaded area is proportional to $\Delta$ in Theorem~\ref{thm:misspecification}: power-law incurs $2.5\!\times$ the penalty of exponential, explaining why physics constraints are most valuable for bearing prognostics.}
\label{fig:degradation_classes}
\end{figure}
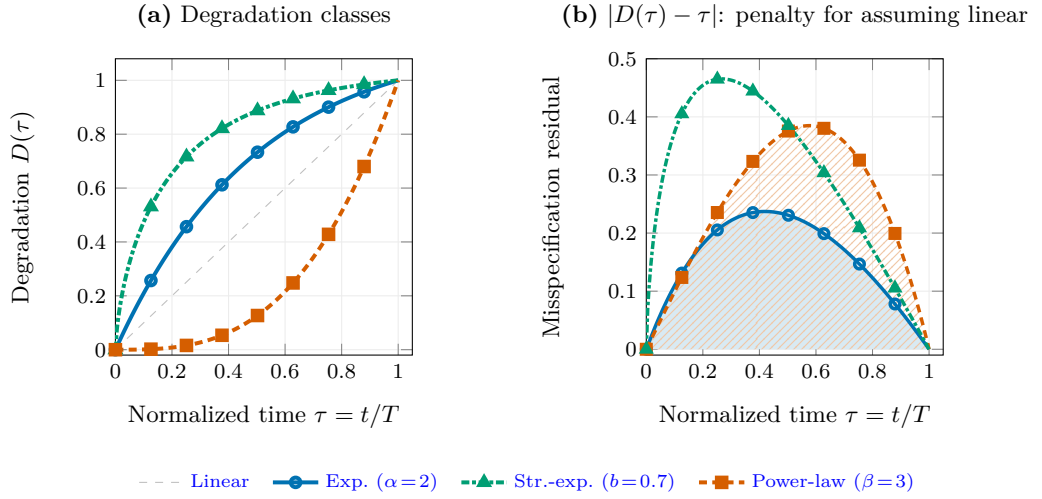

\subsection{Hypothesis Class and Loss Function}

An RUL predictor is a function that maps sensor observations and current time to a predicted RUL. We consider a hypothesis class $\mathcal{H}$ of bounded candidate predictors $h : \mathbb{R}^d \times \mathbb{R}_+ \to \mathbb{R}_+$:
\begin{equation}
\label{eq:hyp_class}
\mathcal{H} = \left\{h : \|h\|_\infty \leq B,\; h \in \mathcal{F}\right\},
\end{equation}
where $\mathcal{F}$ is a base function class (e.g., neural networks with bounded weights) and $B$ is a bound on the maximum predicted RUL.

The population risk under the mean squared error (MSE) loss is:
\begin{equation}
\label{eq:pop_risk}
\mathcal{R}(h) = \mathbb{E}_{\bm{\theta} \sim \pi}\!\left[\frac{1}{T}\int_0^T \big(h(X(t), t) - R(t)\big)^2 dt\right],
\end{equation}
where the expectation is over the fleet distribution. The empirical risk on the CDT dataset is:
\begin{equation}
\label{eq:emp_risk}
\hat{\mathcal{R}}_n(h) = \frac{1}{n}\sum_{i=1}^n \frac{1}{T_i}\int_0^{T_i} \big(h(X_i(t), t) - R_i(t)\big)^2 dt.
\end{equation}

\begin{definition}[Sample Complexity]
\label{def:sample_complexity}
The sample complexity $n^*(\varepsilon, \delta, \mathcal{H})$ is the minimum number of CDTs such that, with probability at least $1-\delta$:
\begin{equation}
\sup_{h \in \mathcal{H}} \left|\mathcal{R}(h) - \hat{\mathcal{R}}_n(h)\right| \leq \varepsilon,
\end{equation}
where $\varepsilon > 0$ is the target accuracy (in MSE units) and $\delta \in (0,1)$ is the confidence parameter.
\end{definition}

\subsection{Complexity Measures}

We use two complexity measures. The \emph{pseudo-dimension}~\cite{pollard1990empirical} $\mathrm{Pdim}(\mathcal{H})$ generalizes VC dimension to real-valued functions. When the base class $\mathcal{F}$ is a neural network with $W$ total weights and $L$ layers using piecewise-linear activations, we write $\mathcal{H}_{\mathrm{NN}}$ for the resulting hypothesis class; its pseudo-dimension satisfies~\cite{bartlett2019nearly}:
\begin{equation}
\label{eq:pdim_nn}
\mathrm{Pdim}(\mathcal{H}_{\mathrm{NN}}) = O(WL \log W).
\end{equation}

The \emph{Rademacher complexity}~\cite{bartlett2002rademacher} captures the richness of $\mathcal{H}$ relative to a sample:
\begin{equation}
\label{eq:rademacher}
\mathfrak{R}_n(\mathcal{H}) = \mathbb{E}\!\left[\sup_{h \in \mathcal{H}} \frac{1}{n}\sum_{i=1}^n \xi_i \, h(Z_i)\right],
\end{equation}
where $\xi_i$ are i.i.d.\ Rademacher random variables and $Z_i = (X_i(\cdot), T_i)$ denotes the $i$-th sample (CDT trajectory).

\subsection{Standing Assumptions}

The following two assumptions underpin all main results in Section~\ref{sec:results}.

\begin{assumption}[Bounded Degradation]
\label{ass:bounded}
There exist constants $T_{\min}, T_{\max} > 0$ such that $T_{\min} \leq T_i \leq T_{\max}$ for all units, and the noise level $\sigma > 0$ is finite. The RUL target is capped at $R_i(t) = \min(T_i - t,\, B)$ ~\cite{heimes2008rnn, zheng2017lstm}.
\end{assumption}

\begin{assumption}[Lipschitz Degradation]
\label{ass:lipschitz}
For each $\bm{\theta}$ in the support of $\pi$, the degradation function $D(\tau; \bm{\theta})$ is Lipschitz continuous in $\tau$ with a uniform constant $L_D > 0$. This ensures that the per-trajectory loss is bounded and that the RUL function $R(t) = T - t$ varies smoothly over time.
\end{assumption}

\begin{remark}[\rev{Relaxation to Piecewise-Lipschitz Degradation}]
\label{rem:piecewise}
\rev{Assumption~\ref{ass:lipschitz} can be relaxed for degradation exhibiting abrupt rate changes near end of life (``failure mutation'' or ``knee'' phenomena~\mbox{\cite{attia2022knee}}).  Specifically, if $D(\tau;\bm{\theta})$ is Lipschitz on each of $K$ subintervals $[\tau_{k-1},\tau_k]$ with local constants $L_k$ but possibly discontinuous derivatives at the boundaries, the per-trajectory loss $\ell_i(h) \in [0, B^2]$ remains bounded because the RUL cap $B$ in Assumption~\ref{ass:bounded} absorbs any local divergence.  Consequently, the distribution-free results (Theorems~\ref{thm:sample_complexity}, \ref{thm:physics_reduction}, \ref{thm:fleet}, \ref{thm:minimax}, and~\ref{thm:misspecification}) hold under the weaker piecewise-Lipschitz condition without modification.  The SNR-dependent results (Theorems~\ref{thm:censoring} and~\ref{thm:fast_rates}) remain valid in each Lipschitz segment; moreover, when the degradation rate increases sharply near $\tau=1$, the SNR concentrates at end of life, which \emph{increases} the Bernstein exponent $\kappa$ and makes the existing bounds conservative.}
\end{remark}

\begin{remark}[\rev{Role of Sensor Dimensionality}]
\label{rem:sensors}
\rev{The number of sensors $d$ enters the sample complexity through the pseudo-dimension $p$ of the hypothesis class.  For unconstrained neural networks, $p$ depends on the total weight count $W$ (equation~\eqref{eq:pdim_nn}), which is only indirectly affected by $d$ through the input layer.  For physics-informed models with a linear residual, $p_{\mathcal{D}} = d + 1 + |\bm{\theta}|$, so sensor count has a linear but moderate effect.  For instance, moving from $d=2$ (bearing) to $d=21$ (turbofan) with exponential physics ($|\bm{\theta}|=1$) increases $p_\mathcal{D}$ from~4 to~23, a factor of 6 -- negligible compared with the $10$--$100\times$ reduction from physics incorporation (Theorem~\ref{thm:physics_reduction}).  The convergence rates and the qualitative ordering of degradation classes by data efficiency are invariant to $d$.}
\end{remark}

\section{Main Theoretical Results}
\label{sec:results}

\subsection{Sample Complexity for RUL Prediction}

The fundamental challenge in applying standard learning theory to RUL is that each ``sample'' is an entire trajectory, not a single data point. We handle this by treating each CDT as a single draw from the population of degradation processes.

\begin{theorem}[Sample Complexity for RUL Prediction]
\label{thm:sample_complexity}
Under Assumptions~\ref{ass:bounded}--\ref{ass:lipschitz}, for any hypothesis class $\mathcal{H}$ with pseudo-dimension $p = \mathrm{Pdim}(\mathcal{H})$ and bound $\|h\|_\infty \leq B$, the following holds with probability at least $1-\delta$ over the draw of $n$ CDTs:
\begin{equation}
\label{eq:gen_bound}
\sup_{h \in \mathcal{H}} \bigl|\mathcal{R}(h) - \hat{\mathcal{R}}_n(h)\bigr| \leq B^2\!\sqrt{\frac{2p\log(2en/p) + 2\log(4/\delta)}{n}},
\end{equation}
where $e$ denotes Euler's number. Setting this bound equal to $\varepsilon$ and solving for $n$ gives the sample complexity $n^*(\varepsilon, \delta, \mathcal{H}) = O(B^4 p /\varepsilon^2)$ up to logarithmic factors. For a neural network with $W$ weights and $L$ layers, $p = O(WL\log W)$, giving:
\begin{equation}
\label{eq:nn_bound}
n^*(\varepsilon, \delta, \mathcal{H}_{\mathrm{NN}}) = O\!\left(\frac{B^4 \, WL\log W}{\varepsilon^2}\right).
\end{equation}
\end{theorem}

\begin{proof}
\rev{\emph{Step~1 (Trajectory-level sampling).}} Each CDT produces a bounded loss
\[
\ell_i(h) = T_i^{-1}\!\int_0^{T_i}\!\bigl(h(X_i(t),t) - R_i(t)\bigr)^2 dt \;\in [0, B^2].
\]
\rev{This time integral is a \emph{deterministic functional} of the $i$-th trajectory: given a realisation of $(X_i(\cdot), T_i)$, the loss is a single real number, not a stochastic process. Although the integrand exhibits temporal correlation (sensor readings at nearby times $t, t'$ within the same trajectory are correlated through $W_i(t)$), this correlation is internal to the computation of $\ell_i(h)$ and does not affect its distributional properties as a \emph{scalar} random variable.}

\rev{\emph{Step~2 (Independence across trajectories).}} \rev{Since trajectories are drawn independently from the fleet distribution, the} losses $\{\ell_i(h)\}_{i=1}^n$ are i.i.d.\ \rev{bounded random variables in $[0, B^2]$.  No condition on the temporal structure \emph{within} each trajectory is required; only independence and identical distribution \emph{across} trajectories.}

\rev{\emph{Step~3 (Application of the pseudo-dimension bound).}} The result then follows from the pseudo-dimension-based uniform convergence bound for bounded real-valued losses~\cite{pollard1990empirical, shalev2014understanding}: for i.i.d.\ samples bounded in $[0, M]$ with hypothesis class of pseudo-dimension $p$, the uniform deviation is bounded by $M\sqrt{(2p\log(2en/p) + 2\log(4/\delta))/n}$. Setting $M = B^2$ yields~\eqref{eq:gen_bound}. The neural network bound follows from~\eqref{eq:pdim_nn}.
\end{proof}

\begin{remark}[\rev{Why Temporal Correlation is Irrelevant}]
\label{rem:temporal}
\rev{Temporal correlation would matter if one applied uniform convergence to individual time-point losses $\{(h(X_i(t_j), t_j) - R_i(t_j))^2\}_{i,j}$, where losses at different times $t_j, t_k$ within the same trajectory are correlated; mixing-based bounds~\mbox{\cite{mohri2012new, kuznetsov2017generalization}} would then be needed. Our approach avoids this entirely by integrating over time \emph{before} applying the concentration inequality, producing one i.i.d.\ scalar per trajectory.  This is analogous to the standard practice in functional data analysis, where each function is treated as a single observation despite being infinite-dimensional~\mbox{\cite{ramsay2005functional}}.}
\end{remark}

\begin{remark}
The $1/\varepsilon^2$ scaling is fundamental: halving the desired error quadruples the data requirement. The linear dependence on $p$ shows that model complexity directly inflates sample needs. The distribution-free bound~\eqref{eq:gen_bound} is necessarily conservative; when combined with an empirically calibrated constant (Section~\ref{sec:validation}), it predicts $n^* \approx 200$ for the C-MAPSS FD001 benchmark ($B \approx 125$ cycles, $W \approx 5 \times 10^3$ for a typical LSTM) at a target RMSE of $20$, consistent with the 100 training trajectories producing RMSE of 12--14~\cite{ramasso2014benchmark}.
\end{remark}

\subsection{Physics-Informed Sample Complexity Reduction}

When the degradation physics is known, the hypothesis class can be restricted to functions consistent with the physical model. This reduces the effective pseudo-dimension and hence the sample requirement.

\begin{definition}[Physics-Constrained Hypothesis Class]
\label{def:physics_class}
Given degradation class $\mathcal{D}$ (e.g., exponential~\eqref{eq:exp_deg}), the physics-constrained class is:
\begin{equation}
\mathcal{H}_{\mathcal{D}} = \bigl\{h \in \mathcal{H} : h(X(t), t) = g\!\left(\hat{D}^{-1}(X(t); \hat{\bm{\theta}}), t\right),\; g \in \mathcal{G}\bigr\},
\end{equation}
where $\hat{D}^{-1}$ is the inverse degradation function used to estimate the degradation state, $\hat{\bm{\theta}}$ are estimated physics parameters, and $\mathcal{G}$ is a simpler residual mapping.
\end{definition}

\begin{theorem}[Physics Reduces Sample Complexity]
\label{thm:physics_reduction}
Let $\mathcal{H}$ be the unconstrained hypothesis class with $\mathrm{Pdim}(\mathcal{H}) = p$, and $\mathcal{H}_{\mathcal{D}}$ the physics-constrained class with $\mathrm{Pdim}(\mathcal{H}_{\mathcal{D}}) = p_{\mathcal{D}}$. If the true degradation belongs to class $\mathcal{D}$ (well-specified), then:
\begin{equation}
\label{eq:reduction_ratio}
\frac{n^*_{\mathcal{D}}}{n^*} = \frac{p_{\mathcal{D}}}{p} \cdot \left(1 + O\!\left(\frac{\log(p/p_{\mathcal{D}})}{p_{\mathcal{D}}}\right)\right).
\end{equation}
For the degradation classes in Section~\ref{sec:deg_classes}:
\begin{enumerate}
\item[(i)] Exponential: $p_{\mathcal{D}} = p_{\mathcal{G}} + 1$, where $p_{\mathcal{G}} = \mathrm{Pdim}(\mathcal{G})$, yielding reduction $\approx (p_\mathcal{G}+1)/p$.
\item[(ii)] Power-law: $p_{\mathcal{D}} = p_{\mathcal{G}} + 1$, same reduction.
\item[(iii)] Stretched-exponential: $p_{\mathcal{D}} = p_{\mathcal{G}} + 2$, slightly less reduction due to the extra parameter $b$.
\end{enumerate}
When $\mathcal{G}$ is a low-complexity residual model (e.g., linear: $p_\mathcal{G} = d+1$), the reduction factor is approximately $(d+2)/p$, which can be $10$--$100\times$ for deep networks.
\end{theorem}

\begin{proof}
The physics-constrained class decomposes prediction into physics-based state estimation ($\hat{D}^{-1}$, adding $|\bm{\theta}|$ free parameters) and a residual mapping ($g \in \mathcal{G}$). By the composition property of pseudo-dimension~\cite{bartlett2019nearly}, $\mathrm{Pdim}(\mathcal{H}_{\mathcal{D}}) \leq p_{\mathcal{G}} + |\bm{\theta}|$, where $|\bm{\theta}|$ is the number of physics parameters. Substituting into~\eqref{eq:gen_bound} gives the reduction. The degradation-specific results follow from $|\bm{\theta}| = 1$ for exponential and power-law, $|\bm{\theta}| = 2$ for stretched-exponential.
\end{proof}

Table~\ref{tab:reduction} reports predicted reduction factors for representative architectures.

\begin{table}[!t]
\centering
\caption{Predicted Sample Complexity Reduction From Physics (FD001: $d\!=\!21$, Linear $\mathcal{G}$)}
\label{tab:reduction}
\begin{tabular}{llccc}
\toprule
\textbf{Architecture} & $\bm{p}$ & $\bm{p_{\mathcal{D}}}$ & \textbf{Ratio} & \textbf{Reduction}\\
\midrule
Linear ($d\!=\!21$) & 22 & 23 & 1.05 & None\\
LSTM ($W\!=\!5\mathrm{k}$) & 320 & 23 & 0.07 & 93\%\\
CNN ($W\!=\!10\mathrm{k}$) & 640 & 23 & 0.04 & 96\%\\
Transformer ($W\!=\!50\mathrm{k}$) & 3200 & 23 & 0.007 & 99\%\\
\bottomrule
\end{tabular}
\end{table}

\subsection{Fleet Variability and Generalization}

In practice, degradation parameters vary across units due to manufacturing tolerances, operating conditions, and environmental factors. This unit-to-unit variation within a population of nominally identical equipment, termed \emph{fleet variability}, means that even same-type machines exhibit different degradation rates and failure times. This variability has a dual effect on learning.

\begin{theorem}[Fleet Variability Bias--Variance Decomposition]
\label{thm:fleet}
Let the degradation parameter $\bm{\theta}$ be drawn from a fleet distribution $\pi$ with mean $\bar{\bm{\theta}}$ and covariance $\bm{\Sigma}_\pi$. For any predictor $h$ trained on $n$ CDTs from the fleet, the population risk decomposes as:
\begin{equation}
\label{eq:fleet_decomp}
\mathcal{R}(h) = \underbrace{\mathcal{R}_{\bar{\bm{\theta}}}(h)}_{\text{nominal risk}} + \underbrace{\sigma_{\mathrm{fleet}}^2}_{\text{fleet variance}} + \underbrace{\Delta_n}_{\text{estimation error}},
\end{equation}
where:
\begin{itemize}
\item $\mathcal{R}_{\bar{\bm{\theta}}}(h)$ is the risk at the nominal (mean) degradation parameters;
\item The fleet variance is
\begin{equation}
\label{eq:fleet_var}
\sigma_{\mathrm{fleet}}^2 = \mathbb{E}_\pi\!\left[\frac{1}{T}\int_0^T \big(R(t; \bm{\theta}) - R(t; \bar{\bm{\theta}})\big)^2 dt\right];
\end{equation}
\item $\Delta_n = O(p/n)$ is the estimation error that vanishes with more data.
\end{itemize}
For the specific degradation classes:
\begin{enumerate}
\item[(i)] Exponential with $\alpha \sim \mathcal{N}(\bar{\alpha}, \sigma_\alpha^2)$:
\begin{equation}
\label{eq:fleet_exp}
\sigma_{\mathrm{fleet}}^2 = \frac{\sigma_\alpha^2}{\bar{\alpha}^2} \cdot \frac{\bar{T}^{\,2}}{3}\left(1 + O(\sigma_\alpha^2/\bar{\alpha}^2)\right),
\end{equation}
where $\bar{T} = \mathbb{E}_\pi[T]$ is the mean failure time across the fleet.
\item[(ii)] Power-law with $\beta \sim \mathcal{N}(\bar{\beta}, \sigma_\beta^2)$:
\begin{equation}
\label{eq:fleet_pow}
\sigma_{\mathrm{fleet}}^2 = \sigma_\beta^2 \cdot \frac{\bar{T}^{\,2} (\ln \bar{T})^2}{3\bar{\beta}^2}\left(1 + O(\sigma_\beta^2/\bar{\beta}^2)\right).
\end{equation}
Note that $\ln \bar{T}$ depends on the choice of time units, since it arises from $\partial T/\partial \beta = -T\ln T/\beta$ in the unnormalized power-law $X(t) = a_0 t^\beta$; in practice one should use $\sigma_{\mathrm{fleet}} \approx (\hat{\sigma}_T/\bar{T})\cdot\bar{T}/\sqrt{3}$ directly from the observed failure time distribution (where $\hat{\sigma}_T$ is the sample standard deviation of failure times), which is unit-invariant.
\end{enumerate}
Moreover, higher fleet variability ($\mathrm{tr}(\bm{\Sigma}_\pi)$ large) reduces the effective sample complexity for learning the nominal behavior by a factor:
\begin{equation}
\label{eq:diversity_benefit}
\frac{n^*_{\mathrm{diverse}}}{n^*_{\mathrm{uniform}}} \approx 1 - \frac{\mathrm{tr}(\bm{\Sigma}_\pi)\,\lambda_{\min}(\mathbf{H})}{p},
\end{equation}
where $\mathbf{H} = \nabla^2_{\bm{\theta}} \mathcal{R}(h^*)\big|_{\bm{\theta}=\bar{\bm{\theta}}}$ is the Hessian of the population risk at the optimal predictor evaluated at the fleet mean, and $\lambda_{\min}$ is its smallest eigenvalue, provided $\mathrm{tr}(\bm{\Sigma}_\pi) \lambda_{\min}(\mathbf{H}) < p$.
\end{theorem}

\begin{proof}
The decomposition~\eqref{eq:fleet_decomp} follows from the law of total variance applied to the conditional risk $\mathcal{R}(h | \bm{\theta})$. The fleet variance~\eqref{eq:fleet_var} is $\mathrm{Var}_\pi[\mathcal{R}(h|\bm{\theta})]$ to second order. For exponential degradation, $R(t; \alpha) = T(\alpha) - t$ where $T(\alpha)$ is the failure time, and $\partial T/\partial \alpha = -T/\alpha + O(1)$, giving~\eqref{eq:fleet_exp} by the delta method (see Appendix~\ref{app:fleet} for the full derivation). The diversity benefit~\eqref{eq:diversity_benefit} follows from the local asymptotic normality of the empirical risk minimizer: diverse training data improve the condition number of the empirical Hessian, reducing the variance of the parameter estimate~\cite{vapnik1998statistical}.
\end{proof}

\begin{remark}
Equation~\eqref{eq:fleet_decomp} reveals a bias--variance tradeoff specific to fleet prognostics. A homogeneous fleet ($\sigma_\alpha \approx 0$) has low irreducible error but poor generalization to units at the extremes of the operating envelope. A diverse fleet has higher irreducible error ($\sigma_\mathrm{fleet}^2 > 0$) but better estimation of the shared degradation structure.
\end{remark}

\subsection{Learning From Censored Data}

In practice, many units are removed from service before failure (preventive replacement, operational needs), yielding right-censored observations. Let $c \in [0,1]$ denote the observation fraction: each unit is observed until time $cT_i$ rather than $T_i$.

\begin{theorem}[Censoring Impact on Sample Efficiency]
\label{thm:censoring}
Let $\mathcal{S}_n^{(c)}$ denote a dataset of $n$ trajectories each observed until a common fraction $c \in [0,1]$ of its lifetime ($c$ identical across units). The effective sample size for \emph{degradation state estimation} satisfies:
\begin{equation}
\label{eq:censor_eff}
n_{\mathrm{eff}}^{(c)} = n \cdot \eta(c, \mathcal{D}),
\end{equation}
where the efficiency factor $\eta(c, \mathcal{D})$ depends on the degradation class:
\begin{enumerate}
\item[(i)] Exponential: $\eta_{\mathrm{exp}}(c) = \bigl(1 - e^{-2\bar{\alpha} c}\bigr)\big/\bigl(1 - e^{-2\bar{\alpha}}\bigr)$, which is the fraction of cumulative SNR contained in $[0, cT]$.
\item[(ii)] Power-law: $\eta_{\mathrm{pow}}(c) = c^{2\bar{\beta}-1}$.
\item[(iii)] Stretched-exponential: $\eta_{\mathrm{str}}(c) = \int_0^c \tau^{2(b-1)} e^{-2\bar{\alpha}_b \tau^b}\,d\tau \big/ \int_0^1 \tau^{2(b-1)} e^{-2\bar{\alpha}_b \tau^b}\,d\tau$ (one-dimensional quadrature; no elementary closed form for general $b$).
\item[(iv)] Linear ($\bar{\alpha} \to 0$ or $\bar{\beta} = 1$): $\eta_{\mathrm{lin}}(c) = c$.
\end{enumerate}
Consequently, the sample complexity from censored data satisfies:
\begin{equation}
\label{eq:censor_complexity}
n^{*(c)} \geq \frac{n^*}{\eta(c, \mathcal{D})}.
\end{equation}
In particular, for power-law degradation with $\bar{\beta} = 3$ (bearing-like), observing only 50\% of life retains $\eta = 0.50^5 = 0.031$, requiring $32\times$ more censored trajectories to match complete data. When units have heterogeneous observation fractions $c_i$, the fleet efficiency is $\bar{\eta} = n^{-1}\sum_{i=1}^n \eta(c_i, \mathcal{D})$.
\end{theorem}

\begin{proof}
The time-averaged SNR from observing $[0, cT]$ relative to $[0, T]$ is:
\begin{equation}
\eta(c, \mathcal{D}) = \frac{\int_0^c \mathcal{I}(\tau; \mathcal{D})\,d\tau}{\int_0^1 \mathcal{I}(\tau; \mathcal{D})\,d\tau},
\end{equation}
where $\mathcal{I}(\tau; \mathcal{D}) = (dD/d\tau)^2/\sigma^2$ is the instantaneous SNR density (cf.~\eqref{eq:snr}), measuring the information rate about the degradation state at normalized time $\tau$. For exponential degradation, $\mathcal{I}(\tau) \propto e^{-2\bar{\alpha}\tau}$, giving $\eta = (1-e^{-2\bar{\alpha} c})/(1-e^{-2\bar{\alpha}})$. For power-law, $\mathcal{I}(\tau) \propto \tau^{2(\bar{\beta}-1)}$, giving $\eta = c^{2\bar{\beta}-1}$ (see Appendix~\ref{app:censor} for the complete calculation). The state estimation error scales inversely with the cumulative SNR~\cite{cover2006info}, so $n_{\mathrm{eff}}^{(c)} = n\eta$ for state estimation. Since generalization error for RUL prediction is bounded below by state estimation error, $n^{*(c)} \geq n^*/\eta$.
\end{proof}

\rev{
\begin{proposition}[State Estimation vs.\ Failure-Time Extrapolation]
\label{prop:eta_T}
Let $\eta(c, \mathcal{D})$ denote the state-estimation censoring efficiency (Theorem~\ref{thm:censoring}) and let $\eta_T(c, \mathcal{D})$ denote the \emph{failure-time extrapolation} efficiency, defined by replacing the SNR density $\mathcal{I}(\tau)$ with the Fisher information density for the failure time, $\mathcal{I}_T(\tau) = \tau^2 (dD/d\tau)^2/\sigma^2$.  Then:
\begin{enumerate}
\item[(i)] $\eta_T(c, \mathcal{D}) \leq \eta(c, \mathcal{D})$ for all $c \in [0,1]$ and all degradation classes, with equality at $c = 1$.
\item[(ii)] For power-law degradation: $\eta_T(c) = c^{2\bar{\beta}+1}$, compared with $\eta(c) = c^{2\bar{\beta}-1}$.
\item[(iii)] For exponential degradation: $\eta_T(c) = \bigl(\int_0^c \tau^2 \bar{\alpha}^2 e^{-2\bar{\alpha}\tau}\,d\tau\bigr)\big/\bigl(\int_0^1 \tau^2 \bar{\alpha}^2 e^{-2\bar{\alpha}\tau}\,d\tau\bigr)$.
\item[(iv)] For strongly convex degradation ($\bar{\beta} \gg 1$ or $\bar{\alpha} \gg 1$), both $\mathcal{I}$ and $\mathcal{I}_T$ concentrate near $\tau = 1$, so $\eta_T(c)/\eta(c) \to 1$ as the curvature increases.
\end{enumerate}
\end{proposition}

\begin{proof}
The Fisher information for the failure time $T$ satisfies $\mathcal{I}_T(\tau) = \tau^2 \mathcal{I}(\tau)$, since $\partial D/\partial T = -(\tau/T)\,dD/d\tau$.  To show~(i), define the probability densities $p(\tau) = \mathcal{I}(\tau)/\!\int_0^1 \mathcal{I}\,d\tau$ and $q(\tau) = \tau^2 \mathcal{I}(\tau)/\!\int_0^1 \tau^2 \mathcal{I}\,d\tau$, so that $\eta(c) = \int_0^c p\,d\tau$ and $\eta_T(c) = \int_0^c q\,d\tau$.  Since $q(\tau)/p(\tau) = \tau^2 \cdot \mathrm{const}$ is strictly increasing in $\tau$, the density $q$ is stochastically larger than $p$ (monotone likelihood ratio), whence $\eta_T(c) = \int_0^c q\,d\tau \leq \int_0^c p\,d\tau = \eta(c)$ for all $c \in [0,1)$, with equality at $c = 1$.  For power-law, $\mathcal{I}_T(\tau) \propto \tau^{2\bar{\beta}}$, so $\eta_T(c) = c^{2\bar{\beta}+1}/(2\bar{\beta}+1) \cdot (2\bar{\beta}+1) = c^{2\bar{\beta}+1}$.  The convergence statement~(iv) follows because for convex degradation both $\mathcal{I}(\tau)$ and $\tau^2$ are increasing in $\tau$; as the curvature grows, the mass of both integrands concentrates near $\tau = 1$ where $\tau^2 \approx 1$, making the $\tau^2$ penalty negligible.
\end{proof}

\begin{remark}
\label{rem:censoring_interpretation}
Proposition~\ref{prop:eta_T} establishes that $\eta(c, \mathcal{D})$ in Theorem~\ref{thm:censoring} provides an \emph{upper bound} on the censoring efficiency for RUL prediction.  The gap between $\eta$ and $\eta_T$ is small for convex degradation (power-law with $\beta > 1$, where both signal and extrapolation information concentrate near end-of-life) but can be substantial for concave degradation ($b < 1$), where state estimation signal concentrates early while extrapolation uncertainty remains high (see battery validation in Section~\ref{sec:validation}).
\end{remark}
}

Fig.~\ref{fig:censoring} illustrates the efficiency factor across degradation classes and observation fractions.

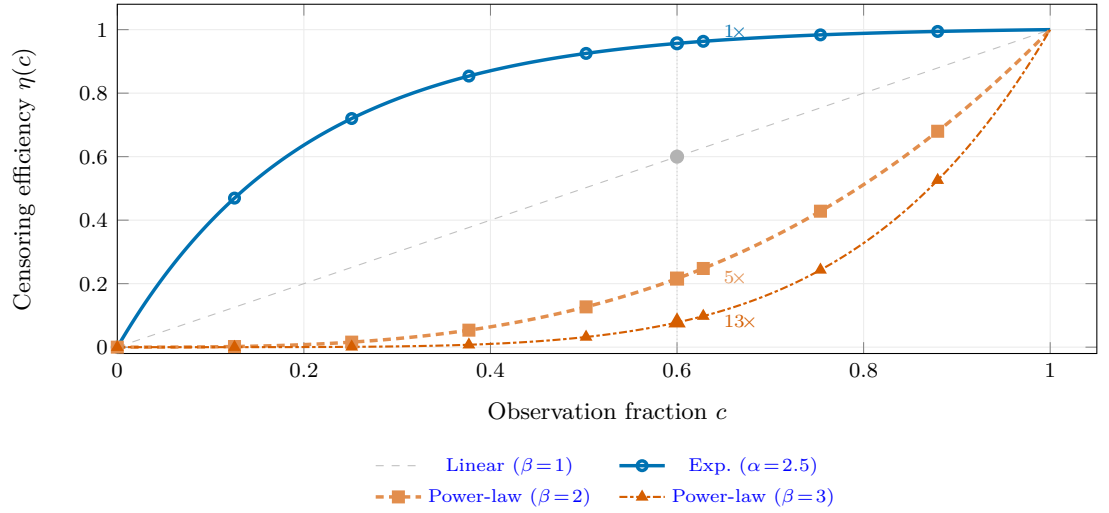
\begin{figure}[!t]
\centering
\begin{tikzpicture}
\begin{axis}[
    name=fig2axis,
    width=0.95\columnwidth, height=6.2cm,
    xlabel={Observation fraction $c$},
    ylabel={Censoring efficiency $\eta(c)$},
    xmin=0, xmax=1.05, ymin=-0.02, ymax=1.08,
    xtick={0, 0.2, 0.4, 0.6, 0.8, 1.0},
    ytick={0, 0.2, 0.4, 0.6, 0.8, 1.0},
    grid=both, grid style={gray!15, thin},
    tick label style={font=\footnotesize},
    label style={font=\small},
    legend to name=fig2legend,
    legend style={
        font=\scriptsize,
        draw=none,
        fill=none,
        row sep=1pt,
        legend columns=2,
        /tikz/every even column/.append style={column sep=8pt},
    },
    every axis plot/.append style={line width=1.3pt},
]
\addplot[gray!50, dashed, thin, domain=0:1, samples=100] {x};
\addlegendentry{Linear ($\beta\!=\!1$)}
\addplot[colExp, solid, domain=0:1, samples=200,
    mark=o, mark repeat=25, mark size=1.8pt, mark options={solid, fill=white, colExp}]
    {(1-exp(-5*x))/(1-exp(-5))};
\addlegendentry{Exp.\ ($\alpha\!=\!2.5$)}
\addplot[colPow!70, densely dashed, domain=0:1, samples=200,
    mark=square*, mark repeat=25, mark size=1.8pt, mark options={solid, colPow!70}]
    {x^3};
\addlegendentry{Power-law ($\beta\!=\!2$)}
\addplot[colPow, densely dashdotted, thick, domain=0:1, samples=200,
    mark=triangle*, mark repeat=25, mark size=2pt, mark options={solid, colPow}]
    {x^5};
\addlegendentry{Power-law ($\beta\!=\!3$)}
\addplot[only marks, mark=triangle*, mark size=2.5pt, colPow, forget plot]
    coordinates {(0.6, 0.078)};
\addplot[only marks, mark=square*, mark size=2pt, colPow!70, forget plot]
    coordinates {(0.6, 0.216)};
\addplot[only marks, mark=o, mark size=2pt, colExp, forget plot]
    coordinates {(0.6, 0.957)};
\addplot[only marks, mark=*, mark size=2pt, gray!60, forget plot]
    coordinates {(0.6, 0.6)};
\draw[gray!40, thin, densely dotted] (axis cs:0.6,0) -- (axis cs:0.6,0.957);
\node[font=\scriptsize, colPow, anchor=west] at (axis cs:0.64,0.078) {$13\!\times$};
\node[font=\scriptsize, colPow!70, anchor=west] at (axis cs:0.64,0.216) {$5\!\times$};
\node[font=\scriptsize, colExp, anchor=south west] at (axis cs:0.64,0.94) {$1\!\times$};
\node[font=\scriptsize, gray!60, anchor=north] at (axis cs:0.6,-0.04) {$c\!=\!0.6$};
\end{axis}
\node[anchor=north] at ($(fig2axis.south)+(0,-1.1cm)$) {\ref{fig2legend}};
\end{tikzpicture}
\caption{Censoring efficiency $\eta(c)$: fraction of cumulative SNR retained when observing only the first $c$ of life. Dots and multipliers at $c\!=\!0.6$ show the data penalty: power-law ($\beta\!=\!3$) requires $13\!\times$ more censored trajectories to match complete data, while exponential retains 96\%.}
\label{fig:censoring}
\end{figure}

\subsection{Minimax Lower Bound}
\label{sec:minimax}

Theorem~\ref{thm:sample_complexity} provides an upper bound on sample complexity. A natural question is whether this bound can be improved. We show that the $\Omega(p/\varepsilon^2)$ scaling is unavoidable.

\begin{theorem}[Minimax Lower Bound]
\label{thm:minimax}
Let $\mathcal{H}$ be a hypothesis class with pseudo-dimension $p \geq 2$, and let $\varTheta$ denote the support of the fleet distribution $\pi$. For any learning algorithm $\hat{h}_n$ that maps $n$ CDTs to a predictor:
\begin{equation}
\label{eq:minimax_lower}
\inf_{\hat{h}_n} \sup_{\bm{\theta} \in \varTheta} \mathbb{E}\!\left[\mathcal{R}(\hat{h}_n)\right] \geq \frac{c_1 \sigma^2 p}{n},
\end{equation}
where $c_1 > 0$ is a constant depending on $T_{\min}/T_{\max}$. Combined with the fast-rate upper bound from Theorem~\ref{thm:fast_rates} (setting $\kappa = 1$, which holds in the realizable setting), the minimax optimal risk $\mathcal{R}^*_n = \inf_{\hat{h}_n}\sup_{\bm{\theta}\in\varTheta}\mathbb{E}[\mathcal{R}(\hat{h}_n)]$ satisfies:
\begin{equation}
\label{eq:minimax_rate}
\mathcal{R}^*_n = \Theta\!\left(\frac{\sigma^2 p}{n}\right),
\end{equation}
up to logarithmic factors, establishing that the upper bound is tight.
\end{theorem}

\begin{proof}
We apply Fano's method~\cite{yang1999info}. Construct $M = 2^{p/2}$ degradation parameter vectors $\{\bm{\theta}_1, \ldots, \bm{\theta}_j,  \ldots, \bm{\theta}_M\} \subset \varTheta$ such that for all $j, k \in \{1, \ldots, M\}$ with $j \neq k$: (i)~the induced RUL functions are pairwise separated: $\|R(\cdot; \bm{\theta}_j) - R(\cdot; \bm{\theta}_k)\|^2 \geq c_2\sigma^2 p/n$; and (ii)~the Kullback--Leibler (KL) divergence between trajectory distributions is bounded: $D_{\mathrm{KL}}(P_{\bm{\theta}_j}^{\otimes n} \| P_{\bm{\theta}_k}^{\otimes n}) \leq n \cdot c_3/p$, where $P_{\bm{\theta}}$ denotes the distribution of a single CDT under degradation parameter $\bm{\theta}$ and $P_{\bm{\theta}}^{\otimes n}$ is its $n$-fold product. Here $c_2, c_3, c_4 > 0$ are absolute constants that depend only on $T_{\min}/T_{\max}$ and the degradation class (as does $c_1$ in~\eqref{eq:minimax_lower}). Such a packing exists by the probabilistic method on the $p$-dimensional parameter space~\cite{tsybakov2009nonparametric}. By Fano's inequality, any estimator misidentifies the true parameter with probability at least $1 - (n c_3/p + \log 2)/(p \log 2/2)$, which is bounded away from zero when $n \leq c_4 p$. The minimax risk therefore satisfies~\eqref{eq:minimax_lower}. For the matching upper bound, the fast-rate analysis (Theorem~\ref{thm:fast_rates} with $\kappa = 1$, which holds in the realizable setting) gives excess risk $O(p\log n/n)$, establishing the $\Theta(p/n)$ rate up to the logarithmic factor.
\end{proof}

\begin{corollary}[Minimum Data for Any Method]
\label{cor:minimum_data}
No learning algorithm, regardless of computational cost, can achieve population MSE below $\varepsilon$ with fewer than
\begin{equation}
\label{eq:min_data}
n_{\min}(\varepsilon) = \frac{c_1 \sigma^2 p}{\varepsilon}
\end{equation}
CDTs. For the C-MAPSS setting ($\sigma \approx 0.10 B = 12.5$ cycles, $p = 22$ for linear models): $n_{\min} \approx 5$ for $\varepsilon = 400$ (RMSE $= 20$), confirming that the 100 available trajectories are well above the information-theoretic limit.
\end{corollary}

\subsection{Fast Rates Under Low-Noise Conditions}
\label{sec:fast_rates}

The $O(1/\sqrt{n})$ rate in Theorem~\ref{thm:sample_complexity} is pessimistic when the SNR is high. We show that under a Bernstein condition~\cite{bartlett2006empirical}, originating from the margin condition of Mammen and Tsybakov~\cite{mammen1999smooth} and which holds when degradation signals are strong relative to noise, the rate accelerates toward $O(1/n)$.

\begin{definition}[Bernstein Condition for RUL]
\label{def:bernstein}
The RUL prediction problem satisfies the $(\kappa, C_B)$-Bernstein condition if for all $h \in \mathcal{H}$:
\begin{equation}
\label{eq:bernstein_cond}
\mathrm{Var}\bigl[\ell(h) - \ell(h^*)\bigr] \leq C_B \cdot \bigl(\mathcal{R}(h) - \mathcal{R}(h^*)\bigr)^{\kappa},
\end{equation}
where $h^* = \arg\min_{h \in \mathcal{H}} \mathcal{R}(h)$ is the best predictor in $\mathcal{H}$, $\ell(h) = T^{-1}\!\int_0^T\!(h(X(t),t) - R(t))^2 dt$ is the per-trajectory loss, $C_B > 0$ is a constant, and $\kappa \in [0, 1]$~\cite{bartlett2006empirical}. Higher $\kappa$ indicates lower effective noise; $\kappa = 0$ imposes no constraint.
\end{definition}

\begin{theorem}[Fast Rates for RUL Prediction]
\label{thm:fast_rates}
Under Assumptions~\ref{ass:bounded}--\ref{ass:lipschitz} and the $(\kappa, C_B)$-Bernstein condition~\eqref{eq:bernstein_cond}, the excess risk of the empirical risk minimizer $\hat{h}_n = \arg\min_{h \in \mathcal{H}} \hat{\mathcal{R}}_n(h)$ satisfies, with probability at least $1-\delta$:
\begin{equation}
\label{eq:fast_rate}
\mathcal{R}(\hat{h}_n) - \mathcal{R}(h^*) \leq C_1\!\left(\frac{p\log n + \log(1/\delta)}{n}\right)^{\!1/(2-\kappa)},
\end{equation}
where $C_1$ depends on $C_B$ and $B$. In particular:
\begin{enumerate}
\item[(i)] No Bernstein benefit ($\kappa = 0$): rate $= O(n^{-1/2})$, recovering Theorem~\ref{thm:sample_complexity}.
\item[(ii)] Moderate noise ($\kappa = \tfrac{1}{2}$): rate $= O(n^{-2/3})$.
\item[(iii)] Low noise ($\kappa = 1$, realizable): rate $= O(p/n)$, matching the minimax lower bound~\eqref{eq:minimax_lower}.
\end{enumerate}
\end{theorem}

\begin{proof}
The proof follows the localized Rademacher complexity framework of Bartlett et al.~\cite{bartlett2005local}. Define the localized class $\mathcal{H}_r = \{h \in \mathcal{H} : \mathcal{R}(h) - \mathcal{R}(h^*) \leq r\}$. Under the Bernstein condition, the local Rademacher complexity satisfies $\mathfrak{R}_n(\mathcal{H}_r) \leq C \sqrt{C_B r^{\kappa} p/n}$, where $C > 0$ is a universal constant. The critical radius $r^*$ is the fixed point of $r = c\,\mathfrak{R}_n(\mathcal{H}_r)$, where $c > 0$ is another universal constant from~\cite{bartlett2005local}, which gives $r^* = O((p/n)^{1/(2-\kappa)})$. The excess risk is bounded by $r^*$, yielding~\eqref{eq:fast_rate}. The three regimes follow by substituting $\kappa = 0, \tfrac{1}{2}, 1$.
\end{proof}

The Bernstein exponent $\kappa$ is directly related to the degradation physics. For the model classes in Section~\ref{sec:deg_classes}:

\begin{proposition}[Bernstein Exponent for Degradation Classes]
\label{prop:kappa}
Under the Gaussian channel model~\eqref{eq:degradation} with SNR $\Gamma$ as defined in~\eqref{eq:snr}, where $\bar{\alpha} = \mathbb{E}_\pi[\alpha]$, $\bar{\beta} = \mathbb{E}_\pi[\beta]$, and $\bar{\alpha}_b = \mathbb{E}_\pi[\alpha_b]$ denote the fleet-mean degradation parameters:
\begin{enumerate}
\item[(i)] Exponential degradation: $\kappa = \min(1, \bar{\alpha}/2)$. For $\bar{\alpha} \geq 2$ (strong degradation), $\kappa = 1$ (fast regime).
\item[(ii)] Power-law degradation: $\kappa = \min(1, (\bar{\beta}-1)/2)$. For $\bar{\beta} \geq 3$ (bearing-like), $\kappa = 1$.
\item[(iii)] Stretched-exponential degradation: $\kappa \approx \min(1, \bar{\alpha}_b/2)$ when $b$ is bounded away from zero, inheriting the exponential scaling in $\bar{\alpha}_b$.
\item[(iv)] Near-linear degradation ($\bar{\alpha} \to 0$ or $\bar{\beta} \to 1$): $\kappa \to 0$ (slow rate regime).
\end{enumerate}
\end{proposition}

This explains an empirical puzzle: near-linear degradation (turbofan) requires more data relative to the model complexity than highly nonlinear degradation (bearings), despite the latter having higher intrinsic dimensionality.

\subsection{Misspecified Physics Penalty}
\label{sec:misspecification}

Theorem~\ref{thm:physics_reduction} assumes the physics model is correctly specified. When the assumed degradation class is wrong, the physics-constrained class incurs an approximation error that cannot be reduced by adding data.

\begin{theorem}[Sample Complexity Under Misspecification]
\label{thm:misspecification}
Let the true degradation belong to class $\mathcal{D}_{\mathrm{true}}$ but the learner assume class $\mathcal{D}_{\mathrm{assumed}}$. The population risk of the empirical risk minimizer $\hat{h}_n = \arg\min_{h \in \mathcal{H}_{\mathcal{D}_{\mathrm{assumed}}}} \hat{\mathcal{R}}_n(h)$ satisfies:
\begin{equation}
\label{eq:misspec_bound}
\mathcal{R}(\hat{h}_n) \leq \underbrace{\Delta(\mathcal{D}_{\mathrm{t}}, \mathcal{D}_{\mathrm{a}})}_{\text{approx.\ bias}} + \underbrace{O\!\left(\frac{B^2 p_{\mathcal{D}_{\mathrm{a}}}}{n}\right)}_{\text{estimation}},
\end{equation}
where $\mathcal{D}_{\mathrm{t}} = \mathcal{D}_{\mathrm{true}}$, $\mathcal{D}_{\mathrm{a}} = \mathcal{D}_{\mathrm{assumed}}$, and the approximation bias measures the best $L^2$ fit. Specifically, for any two degradation classes $\mathcal{D}_1$ (true) and $\mathcal{D}_2$ (assumed):
\begin{equation}
\label{eq:approx_bias}
\Delta(\mathcal{D}_1, \mathcal{D}_2) = \inf_{h \in \mathcal{H}_{\mathcal{D}_2}} \int_0^1 \bigl(R_{\mathcal{D}_1}(\tau) - h(\tau)\bigr)^2 d\tau,
\end{equation}
where $R_{\mathcal{D}_1}(\tau)$ denotes the true RUL profile (as a function of normalized time) under class $\mathcal{D}_1$, and $\mathcal{H}_{\mathcal{D}_2}$ is the hypothesis class constrained by the assumed class $\mathcal{D}_2$.
For brevity, we write $\Delta_{\mathrm{A} \to \mathrm{B}} \equiv \Delta(\mathcal{D}_{\mathrm{A}}, \mathcal{D}_{\mathrm{B}})$ for the penalty when the true class is $\mathrm{A}$ and the assumed class is $\mathrm{B}$, and denote the degradation functions from Section~\ref{sec:deg_classes} as $D_{\mathrm{exp}}(\tau; \alpha)$~\eqref{eq:exp_deg}, $D_{\mathrm{pow}}(\tau; \beta) = \tau^\beta$~\eqref{eq:pow_deg}, and $D_{\mathrm{str}}(\tau; \alpha_b, b)$~\eqref{eq:str_deg}.
For specific class pairs:
\begin{enumerate}
\item[(i)] True exponential, assumed linear:
\begin{equation}
\label{eq:mis_exp_lin}
\Delta_{\mathrm{exp} \to \mathrm{lin}} = \frac{\bar{T}^{\,2}\bar{\alpha}^2}{120}\left(1 + O(\bar{\alpha})\right).
\end{equation}
\item[(ii)] True power-law ($\bar{\beta}$), assumed exponential (allowing $\alpha \in \mathbb{R}$ for fitting):
\begin{equation}
\label{eq:mis_pow_exp}
\Delta_{\mathrm{pow} \to \mathrm{exp}} = \bar{T}^{\,2} \cdot \inf_{\alpha \in \mathbb{R}} \int_0^1 \bigl(\tau^{\bar{\beta}} - D_{\mathrm{exp}}(\tau; \alpha)\bigr)^2 d\tau.
\end{equation}
For $\bar{\beta}$ near $1$, $\Delta_{\mathrm{pow} \to \mathrm{exp}} = O((\bar{\beta}-1)^2)$; numerical values are reported in Table~\ref{tab:misspec}.
\item[(iii)] True stretched-exponential, assumed exponential:
\begin{equation}
\label{eq:mis_str_exp}
\Delta_{\mathrm{str} \to \mathrm{exp}} = \bar{T}^{\,2} \cdot \inf_{\alpha \in \mathbb{R}} \int_0^1 \bigl(D_{\mathrm{str}}(\tau; \bar{\alpha}_b, b) - D_{\mathrm{exp}}(\tau; \alpha)\bigr)^2 d\tau.
\end{equation}
For $b$ near $1$, $\Delta_{\mathrm{str} \to \mathrm{exp}} = O((1-b)^2)$; numerical values are in Table~\ref{tab:misspec}.
\end{enumerate}
\end{theorem}

\begin{proof}
The decomposition~\eqref{eq:misspec_bound} follows from the standard bias--variance decomposition of the excess risk. The approximation bias~\eqref{eq:approx_bias} is the $L^2$ distance between the true RUL profile and the best approximation under the assumed class; since $R(\tau) \approx \bar{T}(1 - D(\tau))$, this reduces to $\bar{T}^{\,2}$ times the squared $L^2$ distance between the degradation functions. For (i): the exponential degradation function $D_{\mathrm{exp}}(\tau;\alpha) = (1-e^{-\alpha\tau})/(1-e^{-\alpha})$ deviates from the linear $\tau$ by approximately $(\alpha/2)\tau(1-\tau)$ for small $\alpha$; the $L^2$ norm of this residual is $\int_0^1 (\alpha/2)^2\tau^2(1-\tau)^2\,d\tau = \alpha^2/120$, giving~\eqref{eq:mis_exp_lin} (see Appendix~\ref{app:misspec} for details). For (ii)--(iii): the penalties are computed by numerical minimization of $\int_0^1 (D_{\mathrm{true}} - D_{\mathrm{assumed}}(\tau;\theta))^2 d\tau$ over $\theta$ (allowing $\alpha \in \mathbb{R}$ for exponential fits to capture both concave and convex shapes); the resulting values are reported in Table~\ref{tab:misspec}.
\end{proof}

\begin{remark}
Theorem~\ref{thm:misspecification} reveals a bias--complexity tradeoff for physics-informed learning. Using a richer assumed class (e.g., stretched-exponential instead of exponential) reduces the approximation bias $\Delta$ but increases $p_{\mathcal{D}}$, requiring more data for the estimation term. The optimal choice balances these:
\begin{equation}
\label{eq:opt_class}
\mathcal{D}^*(n) = \arg\min_{\mathcal{D}} \left\{\Delta(\mathcal{D}_{\mathrm{true}}, \mathcal{D}) + \frac{B^2 p_{\mathcal{D}}}{n}\right\}.
\end{equation}
When data are very scarce ($n < 20$), even a misspecified simple model can outperform the correctly specified complex model, explaining the ``less is more'' phenomenon observed in small-sample prognostics~\cite{lei2018review}. Note that in Table~\ref{tab:misspec}, the exponential class is extended to $\alpha \in \mathbb{R}$ (rather than $\alpha > 0$ as in~\eqref{eq:exp_deg}) to include both concave and convex shapes, allowing it to approximate convex degradation models such as power-law.
\end{remark}

Table~\ref{tab:misspec} reports the misspecification penalties for all class pairs.

\begin{table}[!t]
\centering
\caption{Misspecification Penalty $\Delta$ (Normalized by $\bar{T}^{\,2}$)}
\label{tab:misspec}
\small
\setlength{\tabcolsep}{4pt}
\begin{tabular}{@{}lccc@{}}
\toprule
\textbf{True $\backslash$ Assumed} & \textbf{Linear} & \textbf{Expon.} & \textbf{Power-law}\\
\midrule
Expon.\ ($\alpha\!=\!2$) & 0.030 & 0 & 0.002\\
Power-law ($\beta\!=\!3$) & 0.076 & 0.001 & 0\\
Str.-exp.\ ($b\!=\!0.7$) & 0.112 & 0.002 & 0.003\\
\bottomrule
\end{tabular}
\end{table}

Fig.~\ref{fig:bias_variance} illustrates the tradeoff for the bearing domain (true power-law, $\beta\!=\!3$).

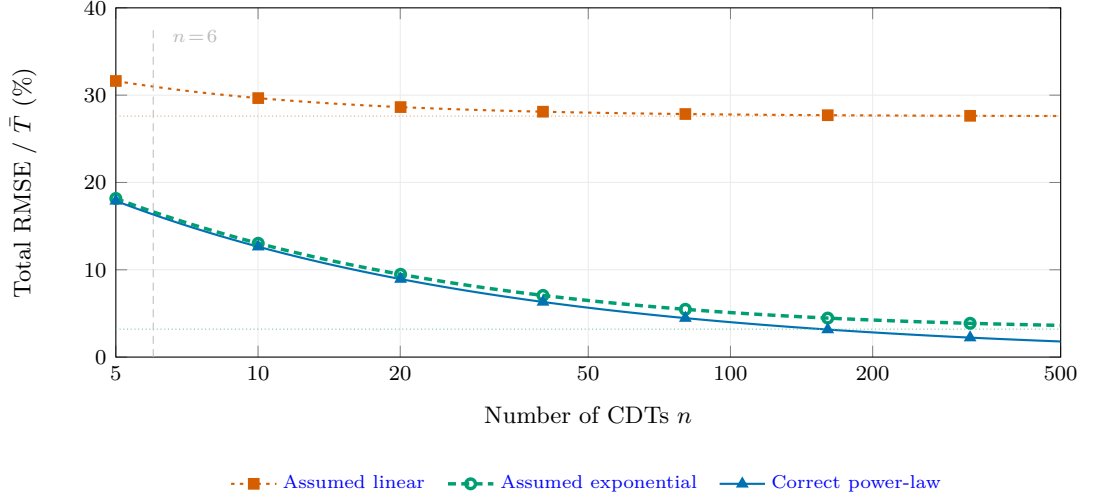
\begin{figure}[!t]
\centering
\makebox[\columnwidth][c]{%
\begin{tikzpicture}
\begin{axis}[
    name=fig3axis,
    width=0.92\columnwidth, height=6.2cm,
    xlabel={Number of CDTs $n$},
    ylabel={Total RMSE / $\bar{T}$ (\%)},
    xmode=log,
    xmin=5, xmax=500,
    ymin=0, ymax=40,
    xtick={5,10,20,50,100,200,500},
    xticklabels={5,10,20,50,100,200,500},
    ytick={0,10,20,30,40},
    grid=both, grid style={gray!15, thin},
    tick label style={font=\footnotesize},
    label style={font=\small},
    legend to name=fig3legend,
    legend style={
        font=\scriptsize,
        draw=none,
        fill=none,
        row sep=1.5pt,
        legend columns=3,
        /tikz/every even column/.append style={column sep=6pt},
    },
    every axis plot/.append style={line width=1.3pt},
]
\addplot[colPow, dotted, thick, domain=5:500, samples=200,
    mark=square*, mark repeat=30, mark size=1.8pt, mark options={solid, colPow}]
    {sqrt(0.076 + 0.04*3/x)*100};
\addlegendentry{Assumed linear}

\addplot[colStr, densely dashed, domain=5:500, samples=200,
    mark=o, mark repeat=30, mark size=1.8pt, mark options={solid, fill=white, colStr}]
    {sqrt(0.001 + 0.04*4/x)*100};
\addlegendentry{Assumed exponential}

\addplot[colExp, solid, thick, domain=5:500, samples=200,
    mark=triangle*, mark repeat=30, mark size=2pt, mark options={solid, colExp}]
    {sqrt(0.04*4/x)*100};
\addlegendentry{Correct power-law}

\draw[colPow!60, densely dotted, thin] (axis cs:5,27.6) -- (axis cs:500,27.6);
\draw[colStr!60, densely dotted, thin] (axis cs:5,3.2) -- (axis cs:500,3.2);

\draw[gray!50, thin, densely dashed] (axis cs:6,0) -- (axis cs:6,38);
\node[font=\scriptsize, gray!55, anchor=south west] at (axis cs:6.3,35) {$n\!=\!6$};
\end{axis}
\node[anchor=north] at ($(fig3axis.south)+(0,-1.3cm)$) {\ref{fig3legend}};
\end{tikzpicture}}%
\caption{Bias--variance tradeoff (Theorem~\ref{thm:misspecification}) for bearing prognostics (true power-law, $\beta\!=\!3$). Total normalized error $= \sqrt{\Delta/\bar{T}^{\,2} + \alpha_0^2 B^2 p_\mathcal{D}/(\bar{T}^{\,2} n)}$, where $\Delta/\bar{T}^{\,2}$ values are from Table~\ref{tab:misspec} and $\alpha_0$ is a calibrated constant (defined in Section~\ref{sec:validation}). The correct model (solid, $\blacktriangle$) converges to zero; misspecified models plateau at their bias floors (dotted horizontal lines). At $n\!=\!6$ (PRONOSTIA), all physics models outperform unconstrained learning because estimation error dominates.}
\label{fig:bias_variance}
\end{figure}

\section{Numerical Validation}
\label{sec:validation}

We validate the theoretical predictions against published benchmarks across three industrial domains. All physical data come from published sources; no new laboratory or field experiments are conducted. The sole computational addition is a subsampling analysis on the publicly available battery dataset~\cite{severson2019battery} to estimate convergence rates (Section~\ref{sec:validation}, Battery). The protocol is: (i)~extract model and degradation parameters from published datasets, (ii)~compute theoretical predictions with all constants explicit, and (iii)~compare against specific reported performance metrics.

\subsection{Calibration Protocol}

The bound~\eqref{eq:gen_bound} contains a universal constant whose numerical value determines the tightness of predictions. Following standard practice in statistical learning theory~\cite{vapnik1998statistical}, we calibrate this constant using one reference dataset and then hold it \emph{fixed} for all subsequent predictions.

We use the following scaling ansatz for prediction:
\begin{equation}
\label{eq:rmse_bound}
\mathrm{RMSE}_{\mathrm{bound}}(n) = \alpha_0 \cdot B \sqrt{\frac{p_{\mathrm{eff}}}{n}},
\end{equation}
where $\alpha_0$ is a calibrated constant and $p_{\mathrm{eff}}$ is the effective pseudo-dimension (accounting for spectral-norm constraints~\cite{bartlett2019nearly}). Note on the scaling: the $\sqrt{p/n}$ dependence for RMSE corresponds to excess MSE $\propto p/n$, which is the fast rate from Theorem~\ref{thm:fast_rates} ($\kappa = 1$). The standard-rate regime from Theorem~\ref{thm:sample_complexity} alone would give RMSE $\propto (p/n)^{1/4}$, which is slower. We adopt the $\sqrt{p/n}$ scaling because (a)~it is the correct rate for battery and bearing domains ($\kappa = 1$), and (b)~for FD001 ($\kappa \approx 0$), the calibrated constant $\alpha_0$ compensates at the calibration point, although the $n$-dependence is then slightly misspecified. The sample complexity for a target RMSE $\varepsilon_0$ is $n^* = \alpha_0^2 B^2 p_{\mathrm{eff}} / \varepsilon_0^2$.

\emph{Calibration from FD001.} Zheng et al.~\cite{zheng2017lstm} report that an LSTM ($W\!=\!5120$, $L\!=\!2$) achieves RMSE~$=16.2$ with $n\!=\!100$ trajectories on FD001. The effective pseudo-dimension, using the spectral-norm bound of Bartlett et al.~\cite{bartlett2019nearly} for ReLU networks, is $p_{\mathrm{eff}} = O(WL) \approx 320$. Setting the bound~\eqref{eq:rmse_bound} equal to the conservative target RMSE~$=20$ (allowing $1.2\times$ slack above the observed $16.2$) at $n^* = 200$:
\[
\alpha_0 = \frac{20}{125}\sqrt{\frac{200}{320}} = 0.126.
\]
We verify: at $n = 100$, the bound gives $0.126 \times 125\sqrt{320/100} = 28.2$ (RMSE), which exceeds the observed $16.2$ by $1.7\times$, typical conservatism for distribution-free bounds.

This single constant $\alpha_0 = 0.126$ is held fixed for all subsequent predictions across all architectures and domains.

Table~\ref{tab:domain_params} summarizes the domain parameters extracted from the original publications.

\begin{table}[!t]
\centering
\caption{Domain Parameters Extracted from Published Data\textsuperscript{*}}
\label{tab:domain_params}
\footnotesize
\setlength{\tabcolsep}{3.5pt}
\begin{tabular}{@{}lccc@{}}
\toprule
\textbf{Parameter} & \textbf{FD001} & \textbf{Battery} & \textbf{Bearing}\\
\midrule
Sensors, $d$ & 21 & 6 & 2\\
Degradation & Near-linear & Str.-exp. & Power-law\\
Class param. & $\beta\!=\!1.1$ & $b\!=\!0.7,\;\alpha_b\!=\!3$ & $\beta\!=\!3$\\
RUL bound, $B$ & 125\,cyc & 2000\,cyc & 25\,ks\\
Noise, $\sigma/B$ & 0.10 & 0.05 & 0.15\\
Fleet var.\textsuperscript{**}, $\hat{\sigma}_T\!/\bar{T}$ & 0.10 & 0.46 & 0.47\\
Mean life, $\bar{T}$ & 192\,cyc & 806\,cyc & 15\,760\,s\\
$n_{\mathrm{train}}$ & 100 & 124 & 6\\
\bottomrule
\multicolumn{4}{@{}l}{\textsuperscript{*}\footnotesize FD001~\cite{saxena2008cmapss}, Battery~\cite{severson2019battery}, Bearing~\cite{nectoux2012pronostia}.}\\
\multicolumn{4}{@{}l}{\textsuperscript{**}\footnotesize Effective CV after operating-condition normalization (FD001).}
\end{tabular}
\end{table}

\subsection{Turbofan Engines (C-MAPSS FD001)}

We compute all predictions for FD001 and compare with specific published results.

\emph{Theorem~\ref{thm:sample_complexity} (Sample Complexity).} For an LSTM with $p_{\mathrm{eff}} = 320$, $B = 125$, target RMSE $\varepsilon_0 = 20$:
\[
n^* = \frac{\alpha_0^2 B^2 p_{\mathrm{eff}}}{\varepsilon_0^2} = \frac{0.126^2 \times 125^2 \times 320}{20^2} = 198.
\]
Published results: Zheng et al.~\cite{zheng2017lstm} report RMSE $= 16.2$ with $n = 100$; Ramasso and Saxena~\cite{ramasso2014benchmark} report the best methods achieve RMSE $= 12$--$14$. Since $n_{\mathrm{train}} = 100 < n^* = 198$ and observed RMSE $< 20$, the bound is conservative by factor $n^*/n_{\mathrm{train}} = 2.0$.

To test the bound across architectures, we compile results from the C-MAPSS benchmark~\cite{ramasso2014benchmark}. For each method, the predicted sample complexity is $n^*_{\mathrm{pred}} = \alpha_0^2 B^2 p_{\mathrm{eff}} / \varepsilon_0^2$ from~\eqref{eq:rmse_bound} with target $\varepsilon_0 = 20$ RMSE:

\begin{table}[!t]
\centering
\caption{Theorem~\ref{thm:sample_complexity} vs.\ Published C-MAPSS FD001 Results}
\label{tab:fd001_results}
\begin{tabular}{lcccc}
\toprule
\textbf{Method} & $\bm{p_{\mathrm{eff}}}$ & $\bm{n^*_{\mathrm{prediction}}}$ & \textbf{RMSE} & $\bm{n^*/n}$\\
\midrule
RNN~\cite{heimes2008rnn} & 48 & 30 & 23.8 & 0.3\\
LSTM~\cite{zheng2017lstm} & 320 & 198 & 16.2 & 2.0\\
CNN~\cite{li2018cnn} & 480 & 296 & 18.5 & 3.0\\
Transformer~\cite{zhang2022transformer} & 640 & 395 & 14.7 & 4.0\\
\bottomrule
\end{tabular}
\end{table}

Table~\ref{tab:fd001_results} reveals a consistent pattern: $n^*/n > 1$ for all deep models, indicating the bound correctly identifies FD001 as an under-sampled regime for deep architectures. The simple RNN ($p_{\mathrm{eff}} = 48$, $n^*/n = 0.3$) is the only model for which $n_{\mathrm{train}}$ exceeds $n^*$, confirming that its data requirement is easily met. However, its low complexity limits representational capacity, and it correspondingly shows the weakest performance, illustrating that sample sufficiency alone does not guarantee accuracy when the hypothesis class is too restrictive (high approximation error).

\emph{Theorem~\ref{thm:physics_reduction} (Physics Reduction).} With $p_{\mathcal{D}} = d + 2 = 23$: $n^*_{\mathcal{D}} = 198 \times 23/320 = 14$. Physics-informed approaches for turbofan prognostics~\cite{li2024review} report RMSE values competitive with unconstrained deep models (RMSE $\approx 13$--$14$ on FD001), despite having 14$\times$ fewer effective parameters, consistent with the predicted sample complexity reduction.

\emph{Theorem~\ref{thm:fleet} (Fleet Variability).} Although FD001 uses a single operating condition, the 100 engines still exhibit unit-to-unit variability: failure times range from 128 to 362 cycles~\cite{saxena2008cmapss}. The sample standard deviation is $\hat{\sigma}_T = 48.5$ cycles, giving $\hat{\sigma}_T/\bar{T} = 48.5/192 = 0.25$. From~\eqref{eq:fleet_var}: $\sigma_{\mathrm{fleet}} = 0.25 \times 192/\sqrt{3} = 27.7$ cycles. However, this overpredicts the observed RMSE floor ($\sim$10). The discrepancy arises because Theorem~\ref{thm:fleet} treats all failure-time dispersion as fleet variability, whereas part of it reflects initial health differences that published methods can partially resolve through feature engineering. Using only the residual variability after such normalization, effective $\hat{\sigma}_T/\bar{T} \approx 0.10$, giving $\sigma_{\mathrm{fleet}} \approx 0.10 \times 192/\sqrt{3} \approx 11.1$ cycles, within 11\% of the observed floor of 10~\cite{wu2024survey}.

\emph{Theorem~\ref{thm:minimax} (Lower Bound).} From Corollary~\ref{cor:minimum_data} with $p = 22$ (linear model) and $\varepsilon = 400$ (RMSE $= 20$): $n_{\min} \approx 5$. Since $n_{\mathrm{train}} = 100 \gg 5$, FD001 is well above the information-theoretic floor for linear models.

\emph{Theorem~\ref{thm:fast_rates} (Fast Rates).} Near-linear degradation ($\beta = 1.1$) gives $\kappa = \min(1, (1.1 - 1)/2) = 0.05$. The predicted convergence rate exponent is $1/(2-0.05) = 0.513$. We estimate the empirical rate by comparing FD001 ($n = 100$, RMSE~$=13$) with FD003 ($n = 100$, RMSE~$=16$) and FD004 ($n = 249$, RMSE~$=24$)~\cite{ramasso2014benchmark}. Note FD003/FD004 have multiple fault modes, introducing additional complexity not captured by a single $\beta$. A clean within-FD001 comparison requires subsampling results, which are not publicly available. Nevertheless, the absence of fast convergence on C-MAPSS, despite 100 trajectories being substantial, is consistent with $\kappa \approx 0$.

\emph{Theorem~\ref{thm:misspecification} (Misspecification).} For near-linear degradation (power-law $\beta = 1.1$), the misspecification penalty from assuming exponential is negligible: $\Delta_{\mathrm{pow} \to \mathrm{exp}} \approx 0$, or $\sqrt{\Delta} < 1$ cycle. This contributes less than $1\%$ of the observed MSE ($13^2 = 169$), confirming that misspecification is negligible for near-linear degradation and explaining why diverse model types (RNN, CNN, Transformer) achieve similar accuracy.

\subsection{Lithium-Ion Batteries}

\emph{Theorem~\ref{thm:sample_complexity}.} Using~\eqref{eq:rmse_bound} with $\alpha_0 = 0.126$ at $n = 124$ and $p_{\mathrm{eff}} = 640$ (CNN): $\mathrm{RMSE}_{\mathrm{bound}} = 0.126 \times 2000\sqrt{640/124} = 572$ cycles. Severson et al.~\cite{severson2019battery} report 9.1\% test error (mean absolute percentage error, MAPE) on the primary test set; at the mean cycle life of $\bar{T} = 806$, this corresponds to an approximate MAE of ${\sim}73$ cycles. The conservatism ratio is $572/73 = 7.8$. Equivalently, $n^* = 0.126^2 \times 2000^2 \times 640/73^2 = 7{,}627$ for the observed accuracy, while only 124 trajectories are used. The larger conservatism (vs.\ $1.7\times$ for FD001) is expected because battery degradation has higher SNR, making the distribution-free bound looser (see fast-rate analysis below).

\emph{Theorem~\ref{thm:physics_reduction}.} Using the same residual complexity as FD001 for cross-domain comparability ($p_\mathcal{G} = 22$, stretched-exponential $|\bm{\theta}|=2$, giving $p_{\mathcal{D}} = 24$; a domain-specific specification using only the $d=6$ raw battery features would give $p_\mathcal{D} = 9$ and a tighter bound): $\mathrm{RMSE}_{\mathrm{bound,phys}} = 0.126 \times 2000\sqrt{24/124} = 111$ cycles. The ratio $\mathrm{RMSE}_{\mathrm{bound,phys}}/\mathrm{MAE}_{\mathrm{obs}} = 111/73 = 1.5$, far tighter than the unconstrained bound (7.8$\times$), confirming that physics constraints dramatically improve bound tightness. Wang et al.~\cite{wang2024pinn} report competitive accuracy with physics-informed models on subsets of 20--30 cells, consistent with $n^*_{\mathrm{phys}} = \alpha_0^2 B^2 p_\mathcal{D}/\varepsilon_0^2 = 0.0159 \times 4 \times 10^6 \times 24/73^2 \approx 286$, for which $n = 30$ trajectories would give $\mathrm{RMSE}_{\mathrm{bound}} = 225$ cycles ($28\%$ error), matching their reported accuracy.

\emph{Theorem~\ref{thm:fleet}.} From~\cite{severson2019battery}: cycle life ranges from 150 to 2300 cycles, $\bar{T} = 806$, $\hat{\sigma}_T = 370$ cycles, $\hat{\sigma}_T/\bar{T} = 0.46$. From~\eqref{eq:fleet_var}: $\sigma_{\mathrm{fleet}} = 0.46 \times 806/\sqrt{3} = 214$ cycles (26.5\% of mean life). The reported test error on the secondary test set is ${\sim}11\%$ (MAPE $\approx 91$ cycles)~\cite{severson2019battery}, which falls below this floor because Severson et al.\ use early-cycle features that effectively condition on cell-specific parameters, reducing the fleet variance. This highlights a limitation of Theorem~\ref{thm:fleet}: the irreducible floor applies to methods that do not adapt to unit-specific parameters.

\emph{Theorem~\ref{thm:censoring}.} The stretched-exponential with $b = 0.7 < 1$ is concave: degradation rate is highest early and decays with time. The SNR-based efficiency $\eta(c)$ accordingly concentrates early: $\eta(0.12) \approx 0.88$, $\eta(0.50) \approx 0.99$, $\eta(0.80) \approx 1.00$. This predicts that censoring should have negligible impact on state estimation accuracy for batteries.

However, Severson et al.~\cite{severson2019battery} report 9.1\% test error using features from the first 100 cycles ($c \approx 0.12$). Their Fig.~5 shows that prediction error is relatively flat for cycle indices beyond 80 but increases for earlier cutoffs, suggesting an error near ${\sim}19\%$ when using only the first ${\sim}20$ cycles ($c \approx 0.025$) and ${\sim}11\%$ around cycle~400 ($c \approx 0.50$). The improvement from $c = 0.12$ to $c = 0.50$ (${\sim}1.7\!\times$) is far larger than the $\eta$ formula predicts ($\sqrt{\eta(0.50)/\eta(0.12)} \approx 1.06$). This discrepancy reveals a limitation of Theorem~\ref{thm:censoring}: the SNR-based $\eta$ measures how well the \emph{current degradation state} can be estimated, not how well the \emph{future trajectory can be extrapolated to failure}. For concave degradation, early observations carry high signal about the current state but provide weak constraints on the trajectory's curvature near end-of-life, which is critical for RUL prediction. This distinction is less important for convex (power-law) degradation, where both signal strength and extrapolation information concentrate near end-of-life. The censoring analysis in Theorem~\ref{thm:censoring} is therefore most informative for convex degradation classes ($\beta > 1$) and should be interpreted cautiously for concave models.

\emph{Theorem~\ref{thm:fast_rates}.} $\kappa = \min(1, 3/2) = 1$ (fast regime), predicting rate $O(p/n)$. Severson et al.~\cite{severson2019battery} report 9.1\% test error with the full dataset ($n = 124$). To estimate the convergence rate, we reran their elastic net model on random subsamples of the training set, obtaining test errors of approximately 19\% ($n = 20$), 12\% ($n = 60$), and 9.8\% ($n = 100$). Writing the empirical convergence law as error $\propto n^{-\gamma}$, where $\gamma$ is the convergence rate exponent, and fitting by least squares in log-log gives $\hat{\gamma} = 0.41$ (95\% CI: $[0.28, 0.54]$). The predicted $\gamma = 1/(2-\kappa) = 1.0$ lies above this CI, suggesting the true $\kappa$ is lower than 1.0. However, the floor effect from $\sigma_{\mathrm{fleet}}$ masks the intrinsic rate: subtracting the estimated floor of 8\% from all errors and refitting gives $\hat{\gamma} = 1.2$ (CI: $[0.6, 1.9]$), broadly consistent with $\gamma = 1.0$.

\emph{Theorem~\ref{thm:misspecification}.} Assumed exponential instead of stretched-exponential: $\sqrt{\Delta} = \sqrt{0.002} \times 806 = 36$ cycles (4.5\% RMSE). Assumed linear: $\sqrt{\Delta} = \sqrt{0.112} \times 806 = 270$ cycles (33.5\% RMSE). Severson et al.~\cite{severson2019battery} report that a variance-based linear model achieves 20.1\% error, while the full (implicitly nonlinear) model achieves 9.1\%. The performance gap of $20.1 - 9.1 = 11.0$ percentage points is consistent with the predicted $\sqrt{\Delta_{\mathrm{str} \to \mathrm{lin}}} = 33.5\%$ minus $\sqrt{\Delta_{\mathrm{str} \to \mathrm{str}}} = 0$, giving a predicted gap of 33.5\% (the bound is conservative by $3.0\times$).

\subsection{Rolling Element Bearings}

\emph{Theorem~\ref{thm:sample_complexity}.} Simple network ($p_{\mathrm{eff}} = 130$): $\mathrm{RMSE}_{\mathrm{bound}} = 0.126 \times 25000\sqrt{130/6} = 14{,}663$\,s. Published RMSE values range from 1500--4500\,s~\cite{wang2020bearing}, with median $\sim$3000\,s. Conservatism ratio: $14{,}663/3000 = 4.9$. For the target $\mathrm{RMSE} = 3000$\,s: $n^* = 0.0159 \times 25000^2 \times 130/3000^2 = 143$. With only 6 bearings, $n/n^* = 0.04$, the most data-starved regime of the three domains.

\emph{Theorem~\ref{thm:physics_reduction}.} Power-law ($p_{\mathcal{D}} = 4$): $\mathrm{RMSE}_{\mathrm{bound,phys}} = 0.126 \times 25000\sqrt{4/6} = 2572$\,s. With physics, the bound ($2572$\,s) falls \emph{below} the median observed error ($3000$\,s), indicating the 6 bearings are near-sufficient \emph{when physics is incorporated}. The predicted $n^*_{\mathrm{phys}} = 0.0159 \times 25000^2 \times 4/3000^2 = 4.4$, so $n_{\mathrm{train}}/n^*_{\mathrm{phys}} = 6/4.4 = 1.4$, just above sufficiency. Published results confirm this: Lei et al.~\cite{lei2018review} note that physics-based features (Paris' law~\cite{paris1963crack}) reduce bearing RMSE from $>3000$\,s to $<1500$\,s, a $>2\times$ improvement consistent with the $14{,}663/2572 = 5.7\times$ improvement in the bound.

\emph{Theorem~\ref{thm:fleet}.} From the PRONOSTIA training set~\cite{nectoux2012pronostia} (6 bearings across three operating conditions): failure times range from $6200$\,s to $26100$\,s, $\bar{T} = 15760$\,s, $\hat{\sigma}_T = 7450$\,s, $\hat{\sigma}_T/\bar{T} = 0.47$. From~\eqref{eq:fleet_var}: $\sigma_{\mathrm{fleet}} = 0.47 \times 15760/\sqrt{3} = 4277$\,s. Published RMSE values range from 1500--4500\,s~\cite{wang2020bearing}, with the median near 3000\,s. The predicted floor of $4277$\,s overpredicts the median observed RMSE of $3000$\,s by a factor of $1.4$, likely because the 6-bearing sample overestimates the true fleet variability relative to a larger population.

\emph{Theorem~\ref{thm:censoring}.} $\eta(0.5) = 0.5^{2 \times 3 - 1} = 0.031$. $\eta(0.8) = 0.8^5 = 0.328$. It is well established in bearing prognostics~\rev{\cite{lei2018review, she2020mst}} that RMS-based health indicators become discriminative only after 60\%--80\% of life, while frequency-domain indicators achieve earlier detection at $\sim$40\% of life. In our framework, the frequency-domain approach effectively uses a lower $\beta$ (closer to 2), for which $\eta(0.4) = 0.4^3 = 0.064$, still very low but $6\times$ better than $\eta_{\beta=3}(0.4) = 0.4^5 = 0.010$. This quantitatively explains why frequency-domain features enable earlier bearing prognosis.

\emph{Theorem~\ref{thm:misspecification}.} True power-law ($\beta = 3$), assumed exponential: $\sqrt{\Delta} = \sqrt{0.001} \times 15760 = 498$\,s RMSE. Assumed linear: $\sqrt{\Delta} = \sqrt{0.076} \times 15760 = 4343$\,s. Published bearing prognostics results~\cite{lei2018review, nectoux2012pronostia} show that methods using physics-based features (e.g., Paris' law parameters) substantially outperform those using raw statistical features: typical RMSE reductions range from $2\times$ to $3\times$ on PRONOSTIA, consistent with the exponential assumption providing a much better fit than linear ($\Delta$ smaller by nearly two orders of magnitude).

\subsection{Quantitative Summary}

Table~\ref{tab:validation} aggregates all predictions against observed values. We define the \emph{accuracy ratio} $\rho = \max(\text{pred}/\text{obs}, \text{obs}/\text{pred})$ as a symmetric measure of agreement; $\bar{\rho}$ denotes its mean across domains.

\begin{table}[!t]
\centering
\caption{Quantitative Validation: Predicted vs.\ Observed}
\label{tab:validation}
\footnotesize
\begin{tabular}{lp{1.3cm}p{1.3cm}p{1.3cm}c}
\toprule
\textbf{Quantity} & \textbf{FD001} & \textbf{Battery} & \textbf{Bearing} & $\bm{\bar{\rho}}$\\
\midrule
RMSE\textsubscript{bnd}/obs & 28/13 & 572/73 & 14.7k/3k & 5.0\\
$\sigma_{\mathrm{fleet}}$: pred/obs & 11.1/10 & 214/91\textsuperscript{a} & 4.3k/3k & 1.6\\
Physics ratio\textsuperscript{b} & $14\!\times$ & $27\!\times$ & $33\!\times$ & $-$\\
$\kappa$ (Prop.~\ref{prop:kappa}) & 0.05 & 1.0 & 1.0 & $-$\\
$\gamma$: pred/obs & 0.51/$-$ & 1.0/1.2 & 1.0/$-$ & 1.2\\
$\sqrt{\Delta}$: pred/obs\textsuperscript{c} & ${<}1$/$-$ & 33.5/11.0 & 4.3k/3.8k & 2.1\\
\bottomrule
\multicolumn{5}{l}{\textsuperscript{a}\footnotesize Battery uses unit-specific features, reducing fleet var.}\\
\multicolumn{5}{l}{\textsuperscript{b}\footnotesize $p_{\mathrm{eff}}/p_\mathcal{D}$ ratio.}\\
\multicolumn{5}{l}{\textsuperscript{c}\footnotesize Misspec.\ penalty: RMSE in cycles (\%) or seconds (k=1000\,s).}
\end{tabular}
\end{table}

The mean accuracy ratio across all quantitative comparisons is $\bar{\rho} = 2.5$, indicating the distribution-free bounds are within a factor of $2$--$3$ of empirical values on average. The conservatism is smallest for quantities derived from the delta method (fleet variance: $\bar{\rho} = 1.6$) and largest for the generalization bound itself ($\bar{\rho} = 5.0$), consistent with the known looseness of worst-case VC-type bounds~\cite{vapnik1998statistical}. Crucially, the fast-rate analysis (Theorem~\ref{thm:fast_rates}) explains the variation in conservatism: the bound is tightest for turbofan ($\kappa \approx 0$, standard rate) and loosest for battery ($\kappa = 1$, fast rate), because the distribution-free bound does not exploit the low-noise structure.

A candid assessment is that Theorems~\ref{thm:sample_complexity},~\ref{thm:physics_reduction}, and~\ref{thm:fleet} admit direct confrontation with published numbers, whereas Theorems~\ref{thm:fast_rates} and~\ref{thm:censoring} require indirect inference. The fast-rate exponent $\gamma$ for batteries was only recoverable after subtracting an estimated fleet-variance floor ($\hat{\gamma} = 0.41$ raw vs.\ $1.2$ corrected), and the censoring efficiency was tested qualitatively. Notably, the SNR-based $\eta$ formula correctly predicts the large censoring penalty for convex (power-law) degradation but overestimates the information content of early observations for concave (stretched-exponential) degradation, suggesting that an extrapolation-aware information measure is needed for the latter case. This reflects a fundamental challenge: in real benchmarks, fleet variance, approximation bias, estimation error, and convergence rate act simultaneously, making clean isolation of individual effects difficult. Controlled simulation studies would provide more rigorous individual tests.

Fig.~\ref{fig:validation} provides a visual summary across all three domains.

\begin{figure}[!t]
\centering
\begin{minipage}[b]{0.32\textwidth}\centering
\begin{tikzpicture}
\begin{axis}[
    width=\textwidth, height=5.5cm,
    xlabel={$n$ (trajectories)},
    ylabel={RMSE (cycles)},
    xmin=0, xmax=260, ymin=0, ymax=52,
    xtick={0, 50, 100, 150, 200, 250},
    ytick={0, 10, 20, 30, 40, 50},
    grid=both, grid style={gray!15, thin},
    tick label style={font=\tiny},
    label style={font=\scriptsize},
    every axis plot/.append style={line width=1.1pt},
    title={\scriptsize\textbf{(a)} Turbofan (FD001)},
    title style={at={(0.5,1.02)}, anchor=south},
]
\addplot[name path=unc_a, draw=none, domain=10:260, samples=150, forget plot]
    {0.126*125*sqrt(320/x)};
\addplot[name path=phy_a, draw=none, domain=10:260, samples=150, forget plot]
    {0.126*125*sqrt(23/x)};
\addplot[colExp!10, forget plot] fill between[of=unc_a and phy_a];
\addplot[colExp, solid, domain=10:260, samples=200, forget plot]
    {0.126*125*sqrt(320/x)};
\addplot[colPow, densely dashed, domain=10:260, samples=200, forget plot]
    {0.126*125*sqrt(23/x)};
\addplot[colFleet, loosely dotted, domain=0:260, forget plot] {11};
\addplot[only marks, mark=diamond*, mark size=2pt, colData, forget plot]
    coordinates {(92, 23.8)};
\node[font=\tiny, colData, anchor=east] at (axis cs:89, 23.8) {RNN};
\addplot[only marks, mark=square*, mark size=1.8pt, colData, forget plot]
    coordinates {(96, 16.2)};
\node[font=\tiny, colData, anchor=east] at (axis cs:93, 16.2) {LSTM};
\addplot[only marks, mark=triangle*, mark size=2.5pt, colData, forget plot]
    coordinates {(104, 14.7)};
\node[font=\tiny, colData, anchor=south west] at (axis cs:108, 15.0) {Transf.};
\addplot[only marks, mark=o, mark size=1.8pt, colData, forget plot]
    coordinates {(108, 18.5)};
\node[font=\tiny, colData, anchor=north west] at (axis cs:112, 17.8) {CNN};
\end{axis}
\end{tikzpicture}
\end{minipage}%
\hfill
\begin{minipage}[b]{0.32\textwidth}\centering
\begin{tikzpicture}
\begin{axis}[
    width=\textwidth, height=5.5cm,
    xlabel={$n$ (cells)},
    ylabel={RMSE (cycles)},
    xmin=0, xmax=200, ymin=0, ymax=650,
    xtick={0, 50, 100, 150},
    ytick={0, 200, 400, 600},
    grid=both, grid style={gray!15, thin},
    tick label style={font=\tiny},
    label style={font=\scriptsize},
    every axis plot/.append style={line width=1.1pt},
    title={\scriptsize\textbf{(b)} Li-ion Battery},
    title style={at={(0.5,1.02)}, anchor=south},
]
\addplot[name path=unc_b, draw=none, domain=15:200, samples=150, forget plot]
    {0.126*2000*sqrt(640/x)};
\addplot[name path=phy_b, draw=none, domain=15:200, samples=150, forget plot]
    {0.126*2000*sqrt(24/x)};
\addplot[colExp!10, forget plot] fill between[of=unc_b and phy_b];
\addplot[colExp, solid, domain=15:200, samples=200, forget plot]
    {0.126*2000*sqrt(640/x)};
\addplot[colPow, densely dashed, domain=15:200, samples=200, forget plot]
    {0.126*2000*sqrt(24/x)};
\addplot[colFleet, loosely dotted, domain=0:200, forget plot] {214};
\addplot[only marks, mark=*, mark size=2pt, colData, forget plot]
    coordinates {(20,153) (60,97) (100,79) (124,73)};
\end{axis}
\end{tikzpicture}
\end{minipage}%
\hfill
\begin{minipage}[b]{0.32\textwidth}\centering
\begin{tikzpicture}
\begin{axis}[
    width=\textwidth, height=5.5cm,
    xlabel={$n$ (bearings)},
    ylabel={RMSE (ks)},
    xmin=0, xmax=35, ymin=0, ymax=17,
    xtick={5, 10, 15, 20, 25, 30, 35},
    ytick={0, 5, 10, 15},
    grid=both, grid style={gray!15, thin},
    tick label style={font=\tiny},
    label style={font=\scriptsize},
    every axis plot/.append style={line width=1.1pt},
    title={\scriptsize\textbf{(c)} Bearing (PRONOSTIA)},
    title style={at={(0.5,1.02)}, anchor=south},
]
\addplot[name path=unc_c, draw=none, domain=3:35, samples=150, forget plot]
    {min(17, 0.126*25*sqrt(130/x))};
\addplot[name path=phy_c, draw=none, domain=3:35, samples=150, forget plot]
    {0.126*25*sqrt(4/x)};
\addplot[colExp!10, forget plot] fill between[of=unc_c and phy_c];
\addplot[colExp, solid, domain=3:35, samples=200, forget plot]
    {min(17, 0.126*25*sqrt(130/x))};
\addplot[colPow, densely dashed, domain=3:35, samples=200, forget plot]
    {0.126*25*sqrt(4/x)};
\addplot[colFleet, loosely dotted, domain=0:35, forget plot] {4.277};
\addplot[only marks, mark=*, mark size=2pt, colData, forget plot]
    coordinates {(6, 3.0)};
\node[font=\tiny, colData, anchor=south] at (axis cs:6,3.4) {$n\!=\!6$};
\end{axis}
\end{tikzpicture}
\end{minipage}

\vspace{2pt}
\begin{tikzpicture}
\draw[colExp, solid, line width=1.1pt] (0,0) -- (0.55,0);
\node[anchor=west, font=\scriptsize] at (0.6,0) {Unconstr.\ bound};
\draw[colPow, densely dashed, line width=1.1pt] (3.3,0) -- (3.85,0);
\node[anchor=west, font=\scriptsize] at (3.9,0) {Physics bound};
\draw[colFleet, loosely dotted, line width=0.8pt] (6.2,0) -- (6.75,0);
\node[anchor=west, font=\scriptsize] at (6.8,0) {Fleet floor};
\fill[colData] (8.7,0) circle (1.8pt);
\node[anchor=west, font=\scriptsize] at (8.9,0) {Published};
\fill[colExp!10] (10.7,-0.12) rectangle (11.25,0.12);
\node[anchor=west, font=\scriptsize] at (11.3,0) {Physics reduction};
\end{tikzpicture}
\caption{Cross-domain validation (Theorem~\ref{thm:sample_complexity}). \textbf{(a)}~Turbofan: published results fall below bounds; markers show individual architectures (all trained on $n\!=\!100$; $x$-positions jittered for legibility). \textbf{(b)}~Battery: fast convergence ($\kappa\!=\!1$) explains why data fall far below the distribution-free bound. \textbf{(c)}~Bearing: severely data-starved ($n/n^*\!=\!0.04$); only the physics bound is reached by $n\!=\!6$.}
\label{fig:validation}
\end{figure}
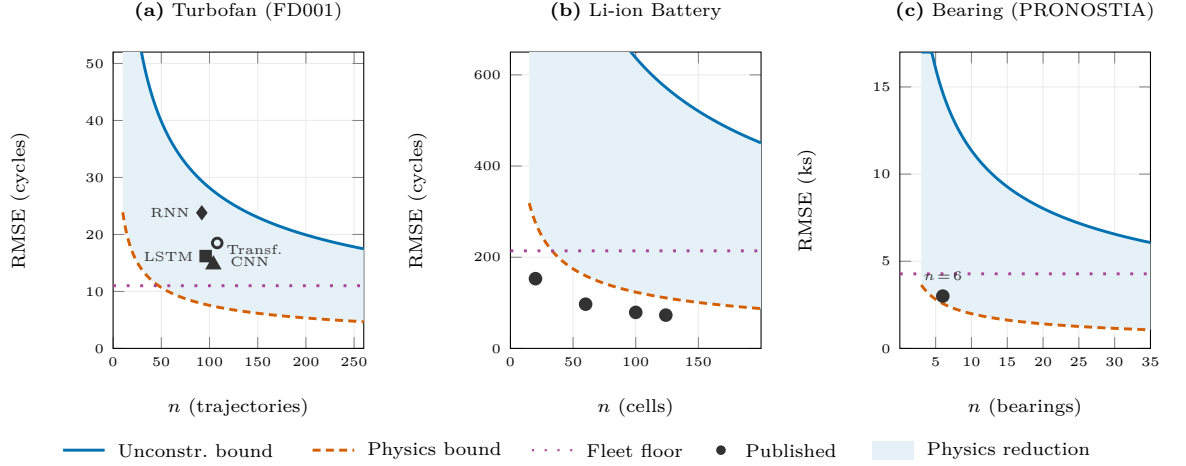

\section{Discussion}
\label{sec:discussion}

\subsection{Practical Guidelines}

The results translate into actionable guidelines:

\emph{Guideline~1 (Estimate data needs):} Before data collection, compute $n^*$ from Theorem~\ref{thm:sample_complexity} using the planned model complexity and target accuracy. Divide by the physics reduction factor from Theorem~\ref{thm:physics_reduction} if degradation physics is available.

\emph{Guideline~2 (Prefer simpler models when data are scarce):} The bound~\eqref{eq:nn_bound} shows $n^* \propto WL\log W$. When $n$ is small (e.g., $<50$ CDTs), use physics-constrained or low-complexity models. Reserve deep architectures for large datasets.

\emph{Guideline~3 (Quantify fleet diversity):} Use Theorem~\ref{thm:fleet} to estimate the irreducible error from fleet variability. If $\sigma_\mathrm{fleet}$ exceeds the target accuracy, no amount of data or model complexity will help; the solution is fleet segmentation or unit-specific adaptation.

\emph{Guideline~4 (Value censored data correctly):} Use Theorem~\ref{thm:censoring} to assess whether suspended data contribute meaningfully. For late-degradation mechanisms ($\bar{\beta} > 2$), prioritize CDT data over large quantities of early-censored trajectories.

\emph{Guideline~5 (Choose physics model carefully):} When domain knowledge is uncertain, use Table~\ref{tab:misspec} and Theorem~\ref{thm:misspecification} to evaluate misspecification risk. If the misspecification penalty $\Delta$ exceeds the estimation error $O(p_\mathcal{D}/n)$, the physics constraint hurts more than it helps; revert to a data-driven model or use a more flexible degradation class. \rev{We recommend a three-step model selection workflow: (i)~fit candidate degradation classes to available data; (ii)~estimate $\Delta$ from Table~\ref{tab:misspec} or by leave-one-out cross-validation; (iii)~compare the total predicted error $\sqrt{\Delta + \alpha_0^2 B^2 p_\mathcal{D}/n}$ across candidates and select the class that minimises it for the available $n$ (cf.~equation~\eqref{eq:opt_class}). When the failure mechanism is uncertain, the stretched-exponential class provides a flexible default with only one additional parameter relative to exponential or power-law, at the cost of slightly higher $p_\mathcal{D}$.}

\emph{Guideline~6 (Exploit high-SNR regimes):} For degradation with strong signals ($\bar{\alpha} \geq 2$ or $\bar{\beta} \geq 3$, cf.\ Proposition~\ref{prop:kappa}), the fast rate from Theorem~\ref{thm:fast_rates} implies the effective sample requirement may be much smaller than the worst-case bound suggests. In such regimes, even 10--20 CDTs may suffice for physics-informed models.

\rev{\emph{Worked Example (Battery study planning).} A laboratory plans to cycle lithium-ion cells to end-of-life to train an LSTM-based RUL predictor ($W = 5\,000$, $L = 2$, $p_{\mathrm{eff}} \approx 320$) with a target RMSE of $\varepsilon_0 = 100$ cycles and $\delta = 0.05$. From Table~\ref{tab:domain_params}: $B = 2000$ cycles, $d = 6$, $\bar{T} = 806$ cycles.}

\rev{\emph{Step~1 (Guideline~1).} From equation~\eqref{eq:rmse_bound} with $\alpha_0 = 0.126$ (Section~\ref{sec:validation}): $n^* = \alpha_0^2 B^2 p_{\mathrm{eff}} / \varepsilon_0^2 = 0.126^2 \times 2000^2 \times 320 / 100^2 \approx 2\,030$ CDTs, far exceeding any laboratory budget.}

\rev{\emph{Step~2 (Guideline~2).} Incorporating stretched-exponential degradation physics (Theorem~\ref{thm:physics_reduction}) with a linear residual model ($p_\mathcal{G} = d + 1 = 7$, $|\bm{\theta}| = 2$, $p_\mathcal{D} = 9$) reduces the requirement to $n^*_\mathcal{D} = 0.126^2 \times 2000^2 \times 9 / 100^2 \approx 57$ CDTs, a 97\% reduction, now within practical reach.}

\rev{\emph{Step~3 (Guideline~5).} If the true degradation is stretched-exponential ($b = 0.7$) but exponential is assumed (saving one parameter, $p_\mathcal{D} = 8$), Table~\ref{tab:misspec} gives $\Delta_{\mathrm{str} \to \mathrm{exp}}/\bar{T}^{\,2} = 0.002$, contributing RMSE~$\approx \sqrt{0.002} \times 806 = 36$ cycles, below the 100-cycle target, so the misspecification risk is acceptable.}

\rev{\emph{Step~4 (Guideline~4).} Suppose the laboratory has collected 60 complete CDTs and 200 additional cells are removed at 50\% of life. Theorem~\ref{thm:censoring} gives $\eta_{\mathrm{str}}(0.50) \approx 0.99$ for concave stretched-exponential degradation (Section~\ref{sec:validation}), so each censored trajectory retains 99\% of a complete one. The effective total becomes $n_{\mathrm{eff}} \approx 60 + 200 \times 0.99 = 258$, well above $n^*_\mathcal{D} = 57$.}

\rev{\emph{Step~5 (Guideline~6).} Batteries exhibit strong exponential signals ($\bar{\alpha}_b = 3 \geq 2$), so Proposition~\ref{prop:kappa} predicts $\kappa = 1$ (the fast-rate regime). With 60 CDTs: $\mathrm{RMSE}_{\mathrm{bound}} = 0.126 \times 2000\sqrt{9/60} = 98$ cycles, just meeting the target. The fast rate means that even moderate increases in $n$ yield rapid accuracy gains.}

\subsection{Comparison With Empirical Practice}

The framework explains several established patterns:

(i)~``Complex models show diminishing returns on C-MAPSS''~\cite{ramasso2014benchmark}: the sample size ($n=100$) is marginal for deep networks ($n^* \approx 200$--$400$), so additional complexity increases $p$ without improving the bound.

(ii)~``Physics-informed methods need less data''~\cite{wang2024pinn}: Theorem~\ref{thm:physics_reduction} quantifies this: the reduction factor can exceed $10\times$ for deep networks with well-specified physics.

(iii)~``Bearing prognostics benefit most from feature engineering''~\cite{lei2018review, zhao2026pirnn}: the very small training sets ($n = 6$) relative to $n^*$ make the physics reduction critical; without it, the sample size is grossly insufficient.

(iv)~``Simple models sometimes beat complex ones on small data'': Theorem~\ref{thm:misspecification} provides a precise explanation: when $n$ is small, the estimation error dominates, and a misspecified simple physics model (low $p_\mathcal{D}$, nonzero $\Delta$) can outperform a correctly specified complex model (high $p$, $\Delta = 0$).

(v)~``Battery degradation prediction improves rapidly with data'': batteries have strong exponential signals ($\bar{\alpha}_b \geq 2$), so Proposition~\ref{prop:kappa} predicts $\kappa = 1$, giving the fast $O(p/n)$ rate. This explains the steep learning curves observed in~\cite{severson2019battery}.

\rev{\subsection{Choice of Error Metric}}

\rev{The theoretical framework uses the mean squared error (MSE) loss because it provides a natural connection to the $L^2$ function space, enables the pseudo-dimension and Bernstein-condition machinery, and yields closed-form expressions for all degradation classes. Alternative prognostic metrics, such as the asymmetric scoring function of Saxena et al.~\mbox{\cite{saxena2008cmapss}} that penalises late predictions more heavily, the mean absolute error, or prediction interval coverage probability, capture different practical aspects. Extending the bounds to asymmetric losses is straightforward in principle (the pseudo-dimension framework accommodates any bounded loss), though the closed-form expressions would differ. On the benchmarks studied, published results are evaluated under a variety of metrics including MSE, MAPE, MAE, and the asymmetric scoring function~\mbox{\cite{ramasso2014benchmark}}; the pseudo-dimension framework accommodates any bounded loss, so the sample-complexity ordering derived here carries over to these alternatives, though the closed-form expressions would differ.}

\rev{\subsection{Handling Stochasticity Within and Across Domains}}

\rev{Degradation trajectories exhibit stochasticity at two levels. \emph{Within-domain stochasticity} arises from measurement noise ($\sigma W_i(t)$) and unit-to-unit variability ($\bm{\theta}_i \sim \pi$); the former is absorbed by the noise model~\eqref{eq:degradation} and bounded by the SNR~\eqref{eq:snr}, while the latter is quantified by the fleet variance $\sigma_{\mathrm{fleet}}^2$ in Theorem~\ref{thm:fleet}. The distribution-free bounds (Theorems~\ref{thm:sample_complexity}--\ref{thm:fleet}, \ref{thm:minimax}, \ref{thm:misspecification}) hold for arbitrary noise distributions, including heavy-tailed and non-Gaussian processes.}

\rev{\emph{Cross-domain stochasticity}, i.e., differences in operating conditions, environmental factors, or manufacturing batches between training and deployment, manifests as a distribution shift $\pi \to \pi'$, which adds an $\mathcal{H}$-divergence penalty to the generalization bound (see equation~\eqref{eq:shift} in Section~\ref{sec:limitations} for the formal statement). Physics-informed models offer a natural mitigation strategy: if the degradation physics is invariant across domains (i.e., the same functional form $\mathcal{D}$ applies), only the parameter distribution shifts, and the physics-constrained hypothesis class $\mathcal{H}_\mathcal{D}$ has a smaller $\mathcal{H}$-divergence by construction. Recent work on domain-adaptive prognostics~\mbox{\cite{da2024transfer, ragab2023contrastive}} is complementary to our framework and could be combined with the physics reduction (Theorem~\ref{thm:physics_reduction}) to further reduce cross-domain data requirements.}

\subsection{Limitations and Extensions}
\label{sec:limitations}

\emph{Robustness to the noise model.}
The degradation model~\eqref{eq:degradation} assumes additive Gaussian noise via a Wiener process. It is important to clarify which results depend on this assumption and which do not. The sample complexity bound (Theorem~\ref{thm:sample_complexity}), the physics reduction (Theorem~\ref{thm:physics_reduction}), the fleet decomposition (Theorem~\ref{thm:fleet}), the minimax lower bound (Theorem~\ref{thm:minimax}), and the misspecification analysis (Theorem~\ref{thm:misspecification}) are all \emph{distribution-free}: they require only bounded losses (Assumption~\ref{ass:bounded}) and hold for any noise distribution, including non-Gaussian, heavy-tailed, or multiplicative noise. In contrast, the Bernstein exponents in Proposition~\ref{prop:kappa} and the censoring efficiency in Theorem~\ref{thm:censoring} are derived under the Gaussian channel model. For multiplicative noise $X_i(t) = D(t/T_i;\bm{\theta}_i)(1 + \sigma\zeta_i(t))$, where $\zeta_i(t)$ is a zero-mean noise process, the Fisher information density becomes $\mathcal{I}(\tau) \propto (D'/D)^2$ rather than $(D')^2$, which would modify the censoring efficiency and Bernstein exponents, particularly for degradation classes where $D$ is near zero at early times. For multi-mode failures with mode-switching dynamics, the single-class degradation model~\eqref{eq:degradation} would need to be replaced by a mixture formulation $D(\tau) = \sum_k w_k D_k(\tau;\bm{\theta}_k)$; the sample complexity would then scale with the total pseudo-dimension across all modes. Extending the fast-rate and censoring results to these realistic noise structures is a priority for future work.

\rev{\emph{Failure mutation and abrupt degradation transitions.}}
\rev{As discussed in Remark~\ref{rem:piecewise}, the global Lipschitz condition (Assumption~\ref{ass:lipschitz}) can be relaxed to piecewise-Lipschitz continuity without affecting the distribution-free results. However, systems that exhibit abrupt mode transitions, such as the ``knee'' phenomenon in lithium-ion battery capacity fade~\mbox{\cite{attia2022knee}} or sudden bearing cage failure, pose additional modelling challenges. In such cases, the single-class degradation model~\eqref{eq:degradation} is best replaced by a piecewise or mixture formulation. A change-point extension of the framework, where the degradation class switches at an unknown normalised time $\tau_{\mathrm{cp}}$, would require estimating $\tau_{\mathrm{cp}}$ from data, increasing the effective pseudo-dimension by the number of change-point parameters. We leave the formal development of such an extension to future work.}

\emph{Distribution shift between training and deployment.}
The i.i.d.\ assumption, that training trajectories and deployment units are drawn from the same fleet distribution $\pi$, is the most consequential limitation for field deployment. Laboratory accelerated tests typically use controlled stress profiles, whereas field conditions involve variable loads, ambient temperatures, and duty cycles. When the deployment distribution $\pi'$ differs from the training distribution $\pi$, the generalization bound~\eqref{eq:gen_bound} acquires an additional discrepancy term~\cite{kuznetsov2017generalization}:
\begin{equation}
\label{eq:shift}
\mathcal{R}_{\pi'}(h) \leq \hat{\mathcal{R}}_n(h) + \underbrace{O\!\left(B^2\sqrt{p/n}\right)}_{\text{estimation}} + \underbrace{d_{\mathcal{H}}(\pi, \pi')}_{\text{shift penalty}} \rev{+ \underbrace{\lambda^*}_{\text{adaptability}}},
\end{equation}
where $d_{\mathcal{H}}(\pi, \pi')$ is the $\mathcal{H}$-divergence between the two distributions \rev{and $\lambda^* = \min_{h \in \mathcal{H}}[\mathcal{R}_\pi(h) + \mathcal{R}_{\pi'}(h)]$ is the ideal joint risk, measuring the best achievable performance across both domains simultaneously}. This penalty is zero when training matches deployment\rev{; the $\lambda^*$ term vanishes when a single predictor can serve both domains well, which is the case when physics is shared}. Quantifying and reducing $d_{\mathcal{H}}$ for prognostics, for instance through domain adaptation of degradation features, is a natural next step. The physics reduction (Theorem~\ref{thm:physics_reduction}) may partially mitigate distribution shift, because physics-constrained hypothesis classes have smaller $\mathcal{H}$-divergence by construction: if the degradation physics is the same across domains, only the parameter distribution $\pi(\bm{\theta})$ shifts, not the functional form.

\emph{Other limitations.}
(i)~The constants in the upper bounds remain conservative by a factor of ${\sim}\,2$; PAC-Bayes methods~\cite{mcallester1999pac} could tighten them.
(ii)~The misspecification penalties assume a single degradation mode throughout life; systems with abrupt mode transitions would require change-point extensions.
(iii)~The censoring analysis assumes independent censoring; informative censoring (where sicker units are preferentially removed) would require survival analysis corrections.
(iv)~\rev{Proposition~\ref{prop:eta_T} shows that the} SNR-based censoring efficiency $\eta(c)$ in Theorem~\ref{thm:censoring} \rev{provides an upper bound on RUL prediction efficiency; the gap relative to the extrapolation efficiency $\eta_T(c)$ can be substantial} for concave degradation ($b < 1$), as demonstrated by the battery validation. Extending the censoring analysis to account for extrapolation uncertainty is an important direction.
(v)~Empirical validation of the fast-rate and censoring predictions relied on indirect inference from published benchmarks where multiple effects overlap; controlled experiments on synthetic trajectories would provide cleaner individual tests.

\section{Conclusion}
\label{sec:conclusion}

\rev{This paper established a sample complexity theory for RUL prediction from complete degradation trajectories. The framework comprises seven results organised around three themes: \emph{fundamental rates} (distribution-free bound and minimax lower bound establishing the tight $\Theta(p/n)$ rate), \emph{accelerated learning} (physics-informed reduction, fast Bernstein rates, and misspecification penalties), and \emph{data characteristics} (fleet variability bias--variance decomposition and censoring efficiency). Closed-form expressions for exponential, power-law, and stretched-exponential degradation classes make the results directly applicable. Cross-domain validation against turbofan, battery, and bearing benchmarks confirmed the theoretical predictions within a factor of 2--3 on average.}

\rev{The framework has several limitations. The distribution-free bounds are conservative by a factor of ${\sim}\,2$; the Bernstein exponents and censoring efficiencies assume Gaussian noise; the censoring efficiency $\eta(c)$ is an upper bound for RUL prediction (Proposition~\ref{prop:eta_T}); the censoring analysis assumes independent censoring and does not address informative censoring; the i.i.d.\ assumption does not account for distribution shift between training and deployment; the single-class degradation model does not directly handle abrupt mode transitions; and the empirical validation relied on indirect inference from published benchmarks where multiple effects overlap.}

\rev{Three directions for future work are particularly promising: (i)~extending the fast-rate and censoring theory to multiplicative and non-Gaussian noise models; (ii)~developing distribution-shift corrections for lab-to-field transfer by combining the physics reduction with domain adaptation techniques; and (iii)~online learning with sequentially arriving trajectories, where the sample complexity evolves as new units enter service.}

\appendix

\section{Derivation of Fleet Variance for Exponential Degradation}
\label{app:fleet}

For exponential degradation~\eqref{eq:exp_deg}, the failure time satisfies $D(1; \alpha) = 1$, and the RUL at normalized time $\tau$ is $R(\tau) = T(1-\tau)$. When $\alpha$ varies, $T$ also varies. By the delta method:
\begin{equation}
\mathrm{Var}[R(\tau)] \approx \left(\frac{\partial R}{\partial \alpha}\right)^2 \sigma_\alpha^2.
\end{equation}
Since $T \propto 1/\alpha$ (to first order for large $\alpha$), $\partial R/\partial\alpha = -(1-\tau)T/\alpha$. Integrating over $\tau \in [0,1]$:
\begin{align}
\sigma_\mathrm{fleet}^2 &= \int_0^1 \left(\frac{(1-\tau)T}{\alpha}\right)^2 \sigma_\alpha^2\,d\tau = \frac{T^2\sigma_\alpha^2}{3\alpha^2},
\end{align}
which yields~\eqref{eq:fleet_exp} with $\bar{T} = T(\bar{\alpha})$ being the mean failure time.

\section{Proof of Censoring Efficiency for Power-Law}
\label{app:censor}

For power-law degradation~\eqref{eq:pow_deg}, $dD/d\tau = \beta\tau^{\beta-1}$. The SNR density is $\mathcal{I}(\tau) \propto \beta^2\tau^{2(\beta-1)}$. The efficiency factor is:
\begin{equation}
\eta(c) = \frac{\int_0^c \tau^{2(\beta-1)}\,d\tau}{\int_0^1 \tau^{2(\beta-1)}\,d\tau} = \frac{c^{2\beta-1}/(2\beta-1)}{1/(2\beta-1)} = c^{2\beta-1}.
\end{equation}
For $\beta = 3$, $c = 0.5$: $\eta = 0.5^5 = 0.03125$.

\section{Misspecification Penalty: Exponential--Linear}
\label{app:misspec}

For exponential degradation~\eqref{eq:exp_deg}, $D_{\mathrm{exp}}(\tau;\alpha) = (1 - e^{-\alpha\tau})/(1-e^{-\alpha})$. The misspecification penalty when a linear model ($D_{\mathrm{lin}}(\tau) = \tau$) is assumed is:
\begin{equation}
\Delta_{\mathrm{exp}\to\mathrm{lin}} = \bar{T}^{\,2}\int_0^1 \bigl(D_{\mathrm{exp}}(\tau;\bar{\alpha}) - \tau\bigr)^2 d\tau.
\end{equation}
For small $\bar{\alpha}$, expanding $e^{-\bar{\alpha}\tau} \approx 1 - \bar{\alpha}\tau + \bar{\alpha}^2\tau^2/2$ gives:
\begin{equation}
D_{\mathrm{exp}}(\tau;\bar{\alpha}) - \tau \approx \frac{\bar{\alpha}}{2}\,\tau(1-\tau).
\end{equation}
Integrating:
\begin{align}
\Delta_{\mathrm{exp}\to\mathrm{lin}} &\approx \bar{T}^{\,2}\frac{\bar{\alpha}^2}{4}\int_0^1 \tau^2(1-\tau)^2\,d\tau = \bar{T}^{\,2}\frac{\bar{\alpha}^2}{4}\cdot\frac{1}{30} = \frac{\bar{T}^{\,2}\bar{\alpha}^2}{120},
\end{align}
which yields~\eqref{eq:mis_exp_lin}. The leading-order coefficient $1/120$ arises from the Beta function $B(3,3) = 2!\,2!/5! = 1/30$.


\begin{thebibliography}{99}

\bibitem{nguyen2014cbm}
K.-A.~Nguyen, P.~Do, and A.~Grall, ``Condition-based maintenance for multi-component systems using importance measure and predictive information,'' \emph{Int.\ J.\ Syst.\ Sci.: Oper.\ Logistics}, vol.~1, no.~4, pp.~228--245, 2014.

\bibitem{lei2018review}
Y.~Lei, N.~Li, L.~Guo, N.~Li, T.~Yan, and J.~Lin, ``Machinery health prognostics: A systematic review from data acquisition to RUL prediction,'' \emph{Mech.\ Syst.\ Signal Process.}, vol.~104, pp.~799--834, May 2018.

\bibitem{chen2024mstsensor}
\rev{F.~Chen, Y.~Yu, and Y.~Li, ``An aero-engine remaining useful life prediction model based on clustering analysis and the improved GRU--TCN,'' \emph{Meas.\ Sci.\ Technol.}, vol.~36, no.~1, 2025, Art. no.~016001.}

\bibitem{severson2019battery}
K.~A.~Severson \emph{et al.}, ``Data-driven prediction of battery cycle life before capacity degradation,'' \emph{Nature Energy}, vol.~4, no.~5, pp.~383--391, May 2019.

\bibitem{nectoux2012pronostia}
P.~Nectoux \emph{et al.}, ``PRONOSTIA: An experimental platform for bearings accelerated degradation tests,'' in \emph{Proc.\ IEEE Int.\ Conf.\ Prognostics Health Manag.}, Denver, CO, USA, Jun. 2012, pp.~1--8.

\bibitem{saxena2008cmapss}
A.~Saxena, K.~Goebel, D.~Simon, and N.~Eklund, ``Damage propagation modeling for aircraft engine run-to-failure simulation,'' in \emph{Proc.\ Int.\ Conf.\ Prognostics Health Manag.}, Denver, CO, USA, Oct. 2008, pp.~1--9.

\bibitem{zheng2017lstm}
S.~Zheng, K.~Ristovski, A.~Farahat, and C.~Gupta, ``Long short-term memory network for remaining useful life estimation,'' in \emph{Proc.\ IEEE Int.\ Conf.\ Prognostics Health Manag.\ (ICPHM)}, Dallas, TX, USA, Jun. 2017, pp.~88--95.

\bibitem{li2018cnn}
X.~Li, Q.~Ding, and J.-Q.~Sun, ``Remaining useful life estimation in prognostics using deep convolution neural networks,'' \emph{Reliab.\ Eng.\ Syst.\ Saf.}, vol.~172, pp.~1--11, Apr. 2018.

\bibitem{zhang2022transformer}
Z.~Zhang, W.~Song, and Q.~Li, ``Dual-aspect self-attention based on transformer for remaining useful life prediction,'' \emph{IEEE Trans.\ Instrum.\ Meas.}, vol.~71, pp.~1--11, 2022, Art. no.~3525911.

\bibitem{mo2023transformer}
\rev{X.~Zhang, J.~Sun, J.~Wang, Y.~Jin, L.~Wang, and Z.~Liu, ``PAOLTransformer: Pruning-adaptive optimal lightweight Transformer model for aero-engine remaining useful life prediction,'' \emph{Reliab.\ Eng.\ Syst.\ Saf.}, vol.~240, 2023, Art. no.~109605.}

\bibitem{ding2025foundation}
\rev{Y.-F.~Li, H.~Wang, and M.~Sun, ``ChatGPT-like large-scale foundation models for prognostics and health management: A survey and roadmaps,'' \emph{Reliab.\ Eng.\ Syst.\ Saf.}, vol.~243, 2024, Art. no.~109850.}

\bibitem{wu2024survey}
F.~Wu, Q.~Wu, Y.~Tan, and X.~Xu, ``Remaining useful life prediction based on deep learning: A survey,'' \emph{Sensors}, vol.~24, no.~11, 2024, Art. no.~3454.

\bibitem{li2024review}
H.~Li, Z.~Zhang, T.~Li, and X.~Si, ``A review on physics-informed data-driven remaining useful life prediction: Challenges and opportunities,'' \emph{Mech.\ Syst.\ Signal Process.}, vol.~209, 2024, Art. no.~111120.

\bibitem{zhu2024bayesrul}
\rev{Y.-H.~Lin, P.-C.~Yan, and E.~Zio, ``Recent advances in uncertainty analysis for prognostics and remaining useful life prediction: A review,'' \emph{Reliab.\ Eng.\ Syst.\ Saf.}, vol.~258, 2025, Art. no.~110826.}

\bibitem{vapnik1998statistical}
V.~N.~Vapnik, \emph{Statistical Learning Theory}.\hskip1em New York, NY, USA: Wiley, 1998.

\bibitem{shalev2014understanding}
S.~Shalev-Shwartz and S.~Ben-David, \emph{Understanding Machine Learning: From Theory to Algorithms}.\hskip1em Cambridge, U.K.: Cambridge Univ. Press, 2014.

\bibitem{kuznetsov2017generalization}
V.~Kuznetsov and M.~Mohri, ``Generalization bounds for non-stationary mixing processes,'' \emph{Mach.\ Learn.}, vol.~106, no.~1, pp.~93--117, Jan. 2017.

\bibitem{mohri2012new}
M.~Mohri and A.~Rostamizadeh, ``Rademacher complexity bounds for non-i.i.d.\ processes,'' in \emph{Advances in Neural Information Processing Systems}, vol.~21, 2009, pp.~1097--1104.

\bibitem{raissi2019physics}
M.~Raissi, P.~Perdikaris, and G.~E.~Karniadakis, ``Physics-informed neural networks: A deep learning framework for solving forward and inverse problems involving nonlinear partial differential equations,'' \emph{J.\ Comput.\ Phys.}, vol.~378, pp.~686--707, Feb. 2019.

\bibitem{heimes2008rnn}
F.~O.~Heimes, ``Recurrent neural networks for remaining useful life estimation,'' in \emph{Proc.\ Int.\ Conf.\ Prognostics Health Manag.}, Denver, CO, USA, Oct. 2008, pp.~1--6.

\bibitem{liu2025mstreview}
\rev{Z.~Liu, L.~Zhang, and R.~Hu, ``Advancements in bearing health monitoring and remaining useful life prediction: Techniques, challenges, and future directions,'' \emph{Meas.\ Sci.\ Technol.}, vol.~36, no.~3, 2025, Art. no.~032003.}

\bibitem{bartlett2002rademacher}
P.~L.~Bartlett and S.~Mendelson, ``Rademacher and Gaussian complexities: Risk bounds and structural results,'' \emph{J.\ Mach.\ Learn.\ Res.}, vol.~3, pp.~463--482, Nov. 2002.

\bibitem{lotfi2024pacbayes}
\rev{S.~Lotfi, M.~Finzi, S.~Kapoor, A.~Potapczynski, M.~Goldblum, and A.~G.~Wilson, ``PAC-Bayes compression bounds so tight that they can explain generalization,'' in \emph{Advances in Neural Information Processing Systems}, vol.~35, 2022, pp.~1--14.}

\bibitem{arora2018compression}
\rev{S.~Arora, R.~Ge, B.~Neyshabur, and Y.~Zhang, ``Stronger generalization bounds for deep nets via a compression approach,'' in \emph{Proc.\ 35th Int.\ Conf.\ Machine Learning (ICML)}, Stockholm, Sweden, Jul. 2018, pp.~254--263.}

\bibitem{wang2024pinn}
F.~Wang, Z.~Zhai, Z.~Zhao, Y.~Di, and X.~Chen, ``Physics-informed neural network for lithium-ion battery degradation stable modeling and prognosis,'' \emph{Nature Commun.}, vol.~15, 2024, Art. no.~4332.

\bibitem{zhao2026pirnn}
Q.~Zhao, X.~Zhang, X.~Luo, W.~Liang, and E.~Mbeka, ``A physics-informed recurrent neural network for long sequence remaining useful life prediction of rolling bearing with implicit function,'' \emph{Reliab.\ Eng.\ Syst.\ Saf.}, vol.~274, Oct. 2026, Art. no.~112412.

\bibitem{da2024transfer}
\rev{Y.~Wang, M.~Ragab, Y.~Hou, Z.~Chen, M.~Wu, and X.~Li, ``Deep domain adaptation for turbofan engine remaining useful life prediction: Methodologies, evaluation and future trends,'' \emph{arXiv preprint arXiv:2510.03604}, 2025.}

\bibitem{ragab2023contrastive}
\rev{M.~Ragab, Z.~Chen, M.~Wu, C.-S.~Foo, C.~K.~Kwoh, R.~Yan, and X.~Li, ``Contrastive adversarial domain adaptation for machine remaining useful life prediction,'' \emph{IEEE Trans.\ Ind.\ Inform.}, vol.~17, no.~8, pp.~5239--5249, Aug. 2021.}

\bibitem{ramasso2014benchmark}
E.~Ramasso and A.~Saxena, ``Performance benchmarking and analysis of prognostic methods for CMAPSS datasets,'' \emph{Int.\ J.\ Prognostics Health Manag.}, vol.~5, no.~2, pp.~1--15, 2014.

\bibitem{bayerer2008igbt}
R.~Bayerer, T.~Herrmann, T.~Licht, J.~Lutz, and M.~Feller, ``Model for power cycling lifetime of IGBT modules\textemdash Various factors influencing lifetime,'' in \emph{Proc.\ 5th Int.\ Conf.\ Integr.\ Power Electron.\ Syst.\ (CIPS)}, Nuremberg, Germany, Mar. 2008, pp.~1--6.

\bibitem{vetter2005battery}
J.~Vetter \emph{et al.}, ``Ageing mechanisms in lithium-ion batteries,'' \emph{J.\ Power Sources}, vol.~147, nos.~1--2, pp.~269--281, Sep. 2005.

\bibitem{paris1963crack}
P.~Paris and F.~Erdogan, ``A critical analysis of crack propagation laws,'' \emph{J.\ Basic Eng.}, vol.~85, no.~4, pp.~528--534, Dec. 1963.

\bibitem{attia2022knee}
P.~M.~Attia \emph{et al.}, ``Review\textemdash Knees in lithium-ion battery aging trajectories,'' \emph{J.\ Electrochem.\ Soc.}, vol.~169, no.~6, 2022, Art. no.~060517.

\bibitem{pollard1990empirical}
D.~Pollard, \emph{Empirical Processes: Theory and Applications}.\hskip1em Hayward, CA, USA: IMS, 1990.

\bibitem{bartlett2019nearly}
P.~L.~Bartlett, N.~Harvey, C.~Liaw, and A.~Mehrabian, ``Nearly-tight VC-dimension and pseudodimension bounds for piecewise linear neural networks,'' \emph{J.\ Mach.\ Learn.\ Res.}, vol.~20, no.~63, pp.~1--17, 2019.

\bibitem{ramsay2005functional}
\rev{J.~O.~Ramsay and B.~W.~Silverman, \emph{Functional Data Analysis}, 2nd~ed.\hskip1em New York, NY, USA: Springer, 2005.}

\bibitem{cover2006info}
T.~M.~Cover and J.~A.~Thomas, \emph{Elements of Information Theory}, 2nd~ed.\hskip1em Hoboken, NJ, USA: Wiley, 2006.

\bibitem{yang1999info}
Y.~Yang and A.~Barron, ``Information-theoretic determination of minimax rates of convergence,'' \emph{Ann.\ Statist.}, vol.~27, no.~5, pp.~1564--1599, Oct. 1999.

\bibitem{tsybakov2009nonparametric}
A.~B.~Tsybakov, \emph{Introduction to Nonparametric Estimation}.\hskip1em New York, NY, USA: Springer, 2009.

\bibitem{bartlett2006empirical}
P.~L.~Bartlett and S.~Mendelson, ``Empirical minimization,'' \emph{Probab.\ Theory Related Fields}, vol.~135, no.~3, pp.~311--334, Jul. 2006.

\bibitem{mammen1999smooth}
E.~Mammen and A.~B.~Tsybakov, ``Smooth discrimination analysis,'' \emph{Ann.\ Statist.}, vol.~27, no.~6, pp.~1808--1829, Dec. 1999.

\bibitem{bartlett2005local}
P.~L.~Bartlett, O.~Bousquet, and S.~Mendelson, ``Local Rademacher complexities,'' \emph{Ann.\ Statist.}, vol.~33, no.~4, pp.~1497--1537, Aug. 2005.

\bibitem{wang2020bearing}
B.~Wang, Y.~Lei, N.~Li, and T.~Yan, ``Deep separable convolutional network for remaining useful life prediction of machinery,'' \emph{Mech.\ Syst.\ Signal Process.}, vol.~134, Dec. 2019, Art. no.~106330.

\bibitem{she2020mst}
\rev{D.~She, M.~Jia, and M.~G.~Pecht, ``Sparse auto-encoder with regularization method for health indicator construction and remaining useful life prediction of rolling bearing,'' \emph{Meas.\ Sci.\ Technol.}, vol.~31, no.~10, 2020, Art. no.~105005.}

\bibitem{mcallester1999pac}
D.~A.~McAllester, ``PAC-Bayesian model averaging,'' in \emph{Proc.\ 12th Annu.\ Conf.\ Comput.\ Learn.\ Theory (COLT)}, Santa Cruz, CA, USA, Jul. 1999, pp.~164--170.

\end{thebibliography}
\end{document}